\documentclass[10pt,onecolumn,letterpaper,final]{article}
\usepackage{cvpr}

\usepackage{newunicodechar}

\usepackage{times}
\usepackage{epsfig}
\usepackage{graphicx}
\usepackage{amsmath}
\usepackage{amssymb}

\usepackage{amsfonts} 
\usepackage{amsthm}
\usepackage{subfigure}
\usepackage{physics}
\usepackage{color}
\usepackage{authblk}

\newtheorem{definition}{Definition}
\newtheorem{theorem}{Theorem}
\newtheorem{assumption}{Assumption}

\newtheorem{lemma}{Lemma}
\newtheorem{remark}{Remark}
\newtheorem{corollary}{Corollary}

\usepackage[pagebackref=true,breaklinks=true,colorlinks,bookmarks=false]{hyperref}

\def\cvprPaperID{****} 
\def\confYear{CVPR 2021}
\cvprfinalcopy
\begin{document}

\title{On the Guaranteed Almost Equivalence between Imitation Learning from Observation and Demonstration}






\author{

    Zhihao Cheng\textsuperscript{\rm 1},
    Liu Liu\textsuperscript{\rm 1},
    Aishan Liu\textsuperscript{\rm 2},
    Hao Sun\textsuperscript{\rm 3},
    Meng Fang\textsuperscript{\rm 4},
    Dacheng Tao\textsuperscript{\rm 1}\\
    $^1$ The University of Sydney\\%
    $^2$ Beihang University\\%
    $^3$ The Chinese University of Hong Kong\\
    $^4$ Tencent Robotics X\\
}
\date{
    $^1$ The University of Sydney\\%
    $^2$ Beihang University\\%
    $^3$ The Chinese University of Hong Kong\\
    $^4$ Tencent Robotics X\\
}

\maketitle

\begin{abstract}
   Imitation learning from observation (LfO) is more preferable than imitation learning from demonstration (LfD) due to the nonnecessity of expert actions when reconstructing the expert policy from the expert data. However, previous studies imply that the performance of LfO is inferior to LfD by a tremendous gap, which makes it challenging to employ LfO in practice. By contrast, this paper proves that LfO is almost equivalent to LfD in the deterministic robot environment, and more generally even in the robot environment with bounded randomness. In the deterministic robot environment, from the perspective of the control theory, we show that the inverse dynamics disagreement between LfO and LfD approaches zero, meaning that LfO is almost equivalent to LfD. To further relax the deterministic constraint and better adapt to the practical environment, we consider bounded randomness in the robot environment and prove that the optimizing targets for both LfD and LfO remain almost same in the more generalized setting. Extensive experiments for multiple robot tasks are conducted to empirically demonstrate that LfO achieves comparable performance to LfD. In fact, most common robot systems in reality are the robot environment with bounded randomness (i.e., the environment this paper considered). Hence, our findings greatly extend the potential of LfO and suggest that we can safely apply LfO without sacrificing the performance compared to LfD in practice.\footnote{Preprint. Under review. Our code and dataset will be made available upon publication.\\
   \indent \ \ Contact: zche3121@uni.sydney.edu.au}
\end{abstract}


\section{Introduction}
Imitation learning (IL), which tries to mimic the expert's behaviors, has been successfully applied to many tasks, \emph{e.g.}, self-driving \cite{bojarski2016end,li2017infogail}, navigation \cite{silver2008high}, and robot locomotion \cite{merel2017learning}. According to the information contained within the expert references, imitation learning can be divided into two categories, Learning from Demonstration (LfD) and Learning from Observation (LfO) \cite{goo2019one,wake2020learning}. Specifically, the only difference between LfD and LfO is the required input data, \emph{i.e.}, the former demands both states and actions while the latter only needs states. 

Intuitively, without actions from the expert, it would be more challenging to duplicate behaviors in the imitation learning scenario. However, motivations for LfO are comparatively strong: (1) it is infeasible to record actions in various practical scenarios. For example, in the robot locomotion task, we often desire robots to clone human behaviors. However, it is significantly difficult for us to obtain human actions (\emph{e.g.}, the forces and torques acting on the joints and actuators \cite{peng2018deepmimic,wang2019imitation,park2019learning}). In these cases, LfD is impracticable due to the absence of human actions. By contrast, LfO could still work effectively by exploiting massive online videos \cite{peng2018sfv}; (2) it is quite unnatural to conduct imitation learning using actions as the guidance. Intelligent creatures primarily learn skills by observing how others accomplish tasks \cite{douglas2006observational}. 
Thus, exploiting clues from actions in imitation learning would somewhat limit the intelligence of agents to a comparatively low level; (3) directly imitating actions is infeasible in situations where the expert and the learner are executing tasks in different environment dynamics \cite{Gangwani2020State-only}. Hence, naively asking the learner to copy actions of the expert would lead to performance degeneration or even control failure. 

Though the three motivations mentioned above illustrate the necessity of LfO, a huge performance gap between LfO and LfD has been reported in previous works \cite{torabi2018behavioral,torabi2018generative,yang2019imitation}. The comparatively low performance of LfO compared to LfD imposes restrictions on LfO's potential in real-world applications. Intuitively, the learning process for LfD is more straightforward since expert references provide the direct action guidance to the imitator. On the contrary, in terms of LfO, the agent is required to figure out the right action to generate state transitions similar to the expert ones. 
This is the conjecture for the existence of the performance disparity \cite{torabi2018generative}. Based on this conjecture, Yang et al. \cite{yang2019imitation} proved that the performance gap is primarily caused by the inverse dynamics disagreement between the imitator and the expert. They further pointed out that LfO is equivalent to LfD when the environment dynamics are injective. However, an abnormal phenomenon occurs that LfO can perform as well as LfD in simple environments that are not injective (\emph{e.g.,} Pendulum-v2) \cite{torabi2018generative,sun2019adversarial}. 
Therefore, investigating the root for the performance gap and bridging the gap between these two methods are significantly beneficial to improve the employment of LfO in practice. In this paper, we prove that the inverse dynamics disagreement between the expert and the imitator approaches zero in deterministic robot tasks or robot tasks with bounded randomness, which means that LfO is almost equivalent to LfD. 

In this paper, we conduct a deep investigation into the difference between LfO and LfD from the control theory. Based on the most representative imitation learning algorithms: Generative Adversarial Imitation Learning (GAIL) \cite{ho2016generative} and Generative Adversarial Imitation from Observation (GAIfO) \cite{torabi2018generative}, we first use the deterministic property of the environment to analyze the value of the inverse dynamics model in a more rigorous way. Then, we establish the Euler-Lagrange dynamical equation for typical Mujoco tasks, and prove the existence and uniqueness of the action for a feasible state transition $(s,s')$. Combing the above two results, we prove that, in the same deterministic robot environment, the inverse dynamics models for the imitator and the expert are almost equal everywhere. 
Thus, the inverse dynamics disagreement between the expert and the imitator approaches zero, meaning that LfO is almost equivalent to LfD in deterministic robot tasks. 
To further relax the deterministic constraint and better adapt to the environment in practice, we consider the bounded randomness in the robot environment and prove that the optimizing targets for LfD and LfO remain almost same in a more generalized setting. 
  Primary contributions of our paper are summarized as follows:
  \begin{itemize}
    \item By exploiting the Euler-Lagrange dynamical equation from control theory, we are the first to prove that in a deterministic robot environment, the performance gap between LfO and LfD approaches zero.
    \item We further prove that the optimizing targets for LfO and LfD remain almost same even if bounded randomness is considered in the robot dynamics, which extends our conclusion to more generalized settings. Moreover, this makes it possible to employ LfO in practical applications.
    \item Extensive experiments on various robot tasks illustrate that the performance of LfO is comparable to that of LfD. Furthermore, we give some discussions and suggestions on the factors which may affect the performance of LfO and LfD.
\end{itemize}

\section{Related Works}
Imitation Learning tries to reproduce a policy that can mimic an expert's specific behaviors using only expert references. From the perspective of the information contained in expert trajectories, imitation learning can be divided into Learning from Demonstration (LfD) and Learning from Observation (LfO) \cite{goo2019one,wake2020learning}.
\paragraph{LfD} LfD utilizes both states and actions to conduct imitation learning and can be divided into two categories, \emph{i.e.}, Behavior Cloning (BC) and Inverse Reinforcement Learning (IRL). Behavior Cloning, as indicated by its name, aims to clone expert's behaviors from expert trajectories which contain states and actions \cite{Bain95,Ross10,Ross11}. Thus, BC solves the imitation problem in a supervised learning manner, taking states as the inputs and actions as the labels. Given current states, the agent learns to predict actions as close as the expert ones. Inverse Reinforcement Learning imitates the expert from the perspective of reward shaping. IRL first reconstructs a reward function from expert trajectories, and then uses this reward function to guide a standard reinforcement learning process. A recent advance in IRL is Generative Adversarial Imitation Learning (GAIL) \cite{ho2016generative}, which makes no assumption on the form of the reward function. Instead, it utilizes a discriminator to measure the similarity of state-action pairs between the imitator and the expert and takes the similarity as the reward to perform forward reinforcement learning.
\paragraph{LfO} LfO is developed in a scenario where expert actions are absent. Similar to BC and GAIL in LfD, corresponding LfO versions are proposed, \emph{i.e.}, Behavior Cloning from Observation (BCO) \cite{torabi2018behavioral} and Generative Adversarial Imitation from Observation (GAIfO) \cite{torabi2018generative}. BCO employs an inverse dynamics model to guess the possible expert actions, and then uses the guessed expert actions to conduct standard BC. Unlike GAIL which uses state-action pairs to generate the reward, GAIfO exploits state transitions to obtain the reward. Generally, the performance of GAIfO is inferior to that of GAIL due to the lack of direct action guidance. Another approach in LfO is to design a hand-crafted reward function with expert states and then employ ordinary reinforcement learning to maximize the episode cumulative reward \cite{peng2018deepmimic,peng2018sfv,peng2020learning}. However, it is not an easy task to design a proper hand-crafted reward function since there is no mature design paradigm for the reward function and it requires background knowledge from the given application field. 
\paragraph{Relationship between LfD and LfO} 
Several studies have empirically shown that LfO is more difficult than LfD \cite{torabi2018behavioral,torabi2018generative}. In other words, the performance of LfO is inferior to that of LfD. An intuitive explanation is that, expert references in LfD are able to teach the imitator the correct actions directly. By contrast, in LfO, the imitator is required to find out the right actions and generates the state-transitions that are similar to those of the expert via interacting with the environment. Evidently, this process brings higher complexity and difficulty. Specifically, Yang et al. \cite{yang2019imitation} investigated the reason for the performance gap between LfD and LfO using two representative IL algorithms, GAIL and GAIfO. 
They believed that, in the complex real-world environment, the inverse dynamics disagreement is not zero, implying that a performance gap exists between LfD and LfO. They further presented a corollary that LfD is equivalent to LfO in injective dynamics. However, according to the definition of injectivity, injective dynamics indicate that a distinct state transition $(s,s')$ in its domain maps to a distinct action $a$ in its codomain. This is quite a harsh requirement and we could not find corresponding environments in practice.
 
Unlike previous works, we suggest that the performance gap between LfD and LfO approaches zero in the deterministic robot system (a class of common system). In particular, we do not require the environment dynamics to be injective. By contrast, we only need that there only exists a unique action $a$ for a feasible state transition $(s,s')$. But on the other hand, an action $a$ could correspond to many distinct state transitions, which means that the environment dynamics are not injective. As a result, our work tremendously relaxes the pre-requisite conditions for the equivalence between LfD and LfO, and extends the possible application domains of LfO in practice.

\section{Preliminaries}
To derive our theorem, prerequisite definitions and concepts are presented in advance. 
\paragraph{Markov Decision Process} We consider a Markov Decision Process described by $(S,A,T,R,\gamma)$, where $S$ and $A$ represent state space and action space respectively, $T=T(s'|s,a)$ is the environment dynamics modeling the probability of states transitions over actions, $R:S \times A \rightarrow R$ is the reward function, and $\gamma$ is the discount factor. Let $\pi(a_t|s_t):S\times A\rightarrow[0,1]$ be a stochastic policy for the agent where $t$ is the current timestep, $J(\pi)=\mathbb{E}_{\mathbf{s}_{0},\mathbf{a}_{0},\cdots}\left[R_0\right]$ in which $R_t=\sum_{l=0}^{\infty}\gamma^l r_{t+l}$ denoting the expected discounted reward, where $\mathbf{s}_0\sim\rho_0(\mathbf{s}_0)$, $a_t\sim\pi(a_t|s_t)$,  $\mathbf{s}_{t+1}\sim T(\mathbf{s}_{t+1}|\mathbf{s}_t,\mathbf{a}_t)$, and $\rho_0(\mathbf{s}_0)$ is the probability distribution of the initial state $s_0$. The goal of Reinforcement Learning (RL) algorithms is to find the optimal policy $\pi^{*}(a_t|s_t)$, which can achieve the maximum episode cumulative reward $J^*(\pi)$. \par
To measure the discrepancy between trajectories, we bring in the following definition which characterizes the property.  
\begin{definition}[Occupancy Measure]\cite{puterman1994markov,torabi2018generative}
Define the state occupancy measure for a distribution as follows:
\begin{align*}
    \rho_{\pi_\theta}(s) = \sum_{t=0}^{\infty} \gamma^t P(s_t = s|\pi_{\theta}),
\end{align*}
where $\gamma$ is a discount and $\pi_{\theta}(a|s)$ stands for the probability of a policy taking action $a$ given state $s$. For briefness, we omit the parameter of a policy and denote $\pi_{\theta}$ as $\pi$.
\end{definition}
In order to model the relationship between the state transitions and the actions, the inverse dynamics model is defined below. 
\begin{definition}[Inverse Dynamics Model]\cite{spong1990adaptive}\label{definition:inversedynamicsmodel}
Given a policy $\pi$ and the dynamics of the environment $T(s'|s,a)$, we can define the density of the inverse dynamics model: 
\begin{align}\label{equation:inversedynamicsmodel}
    \rho_{\pi}(a|s,s') = \frac{T (s'|s,a)\pi(a|s)}{\int_{A} T (s'|s,\bar a)\pi(\bar a|s) d\bar a}.
\end{align}
\end{definition}

\paragraph{GAIL} GAIL tries to minimize the divergence between the expert trajectory and the agent trajectory. When the trajectory generated from the agent's current policy matches that of the expert, it suggests that the trained policy can achieve similar performance to the expert. The type of divergence could be Kullback–Leibler divergence \cite{kullback1951information} or Jensen–Shannon divergence \cite{lin1991divergence}, which measures the distribution distance between trajectories. It is formalized as follows:
\begin{align*}
    \begin{split}
        \mathop{\min}\limits_{\theta} \mathop{\max}\limits_{\omega} \mathop{\mathbb{E}}\limits_{s,a \sim \rho_\theta}[\log D_w(s,a)]+\mathop{\mathbb{E}}\limits_{s,a \sim \rho_{E}}[\log(1- D_w(s,a))],
    \end{split}
\end{align*}
where $D_w(s,a)$ is a discriminator which judges the similarity of a state-action pair $(s,a)$ to the reference one, and $w$ represents the weights of the discriminator network. 
\paragraph{GAIfO} GAIfO is a modified version of GAIL, which only uses the observation pair $(s,s')$ to learn. Its objective is defined as follows:
\begin{align*}
    \begin{split}
        \mathop{\min}\limits_{\theta} \mathop{\max}\limits_{\omega} \mathop{\mathbb{E}}\limits_{s, s' \sim \rho_\theta}[\log D_w(s,s')]+\mathop{\mathbb{E}}\limits_{s, s' \sim \rho_{E}}[\log(1- D_w(s, s'))].
    \end{split}
\end{align*}
However, according to the literature \cite{torabi2018generative,yang2019imitation}, GAIfO performs worse than GAIL despite its merit of not requiring expert actions. An insight into the performance gap is presented below.

\paragraph{Inverse Dynamics Disagreement between LfD and LfO} 
Yang et al. \cite{yang2019imitation} investigated the cause for the performance gap and revealed this cause with the following equation:
\begin{align*}
\begin{split}
    \mathbb{D}_{KL}(\rho_{\pi}(a|s,s')||\rho_{E}(a|s,s'))=\mathbb{D}_{KL}(\rho_{\pi}(s,a)||\rho_{E}(s,a)) - \mathbb{D}_{KL}(\rho_{\pi}(s,s')||\rho_{E}(s,s')),
\end{split}
\end{align*}
where $\mathbb{D}_{KL}(\rho_{\pi}(a|s,s')||\rho_{E}(a|s,s'))$ is the inverse dynamics disagreement which induces the performance gap between LfD and LfO. Furthermore, they proved that LfD and LfO should be equivalent when the environment dynamics are injective. In this paper, we would further analyze the inverse dynamics disagreement in a stricter manner by utilizing the property of robot environment dynamics, and our analysis prove that the equivalence of LfD and LfO holds even when the environment dynamics are not injective.

\section{Method}
In this section, we will present our theoretical results---the almost equivalence between LfD and LfO in both deterministic robot environments and bounded randomness robot environments, which are organized into three subsections. More precisely, firstly, some properties of the deterministic robot environment and the Euler-Lagrange Dynamical Equations are presented; then, we prove that on deterministic robot tasks, LfD and LfO are almost the same in terms of the optimizing target, which means that the performance gap between them approaches zero; finally, by exploiting the bounded randomness of the environment and the Lipshitz continuity of policies, the almost equivalence between LfO and LfD is still guaranteed. Theoretically, we can see that LfO could perform as well as LfD essentially, and our findings increase the potential for LfO to be applied in real-world robot systems. In the next section, the experimental results validate our theoretical analysis. \par
\subsection{Deterministic Robot Environments}\label{subsection:deterministic}
For LfD and LfO considered in Ho et al. \cite{ho2016generative} and Torabi et al. \cite{torabi2018generative}, respectively, the learner policy is interacting with the same environment in which the expert reference is sampled. That is to say, we do not consider environment variations in the imitation learning process, and this setting is consistent with previous researches showing the performance discrepancy. Thus, to formalize this common situation in imitation learning, an assumption is made consequently. 
\begin{assumption}\label{assum:samedynamics}
The environment dynamics for the learner and the expert remain same. 
\end{assumption}
This assumption does make sense in reality because it is the most basic task to demand the learner to imitate the expert without environment dynamics change and morphology inconsistency \cite{ho2016generative,sun2019adversarial}. Furthermore, in this subsection, we focus on one kind of more specific environment, \emph{e.g.}, the deterministic robot environment, which is the most fundamental setting for robots and is of significance to investigate the difference between LfO and LfD given such environment. 
The deterministic property means that there is no randomness in the environment, leading to the following lemma. 
\begin{lemma}\label{lemma:dynamics}
For a deterministic environment, given the current state $s$, taking action $a$ will result in a certain next state $s'$. To distinguish this certain next state against the others, we denote it as $s_{d}'$ in which the subscript $d$ stands for determinacy. Consequently, the transition dynamics for both the imitator and the demonstrator can be further specified as:
\begin{align*}
T (s'|s,a)=\left\{
\begin{array}{rcl}
1.0       &      & {s'=s_{d}'  }\\
0.0     &      & {s'\in S \ and \ s'\neq s_{d}'}
\end{array} \right .
\end{align*}
\end{lemma}
The zero-one property of the environment dynamics can help us simplify the inverse dynamics model in Eq. \eqref{equation:inversedynamicsmodel}. Besides, to present another feature of the deterministic robot environment, we would borrow the idea from the field of control science. In MDP, $T(s'|s,a)$ is employed to describe the environment dynamics, \emph{i.e}., the dynamical relationship between forces and motions. In contrast to $T(s'|s,a)$ in MDP, the control science community uses differential equations to represent this relationship. Under Lagrangian mechanics \cite{brizard2nd}, we obtain the dynamics model in control science as follows:
\begin{lemma}\label{lemma:lagrange} \cite{dubowsky1993kinematics,westervelt1st,nanos2012cartesian}
The differential equation, called Euler-Lagrange dynamical equation, which models the dynamics for deterministic robot environments, could be written in the following form:
\begin{align}\label{equation:Euler-La}
    M(q)\ddot q + C(q,\dot q)\dot q + G(q)= u,
\end{align}
where $q$ is the generalized coordinates, $\dot q$ is the derivative of the generalized coordinates, and $u$ is the control input. $M(q)$ is the inertia matrix which is positive definite, $C(q,\dot q)$ describes the centripetal and Coriolis torques, and $G(q)$ is the vector of gravitational torques. 
\end{lemma}
From the above lemma, Eq. \eqref{equation:Euler-La} is able to describe the system dynamics of the robot. Besides, both $q$ and $\dot q$ are defined as system states in control theory. However, to distinguish the states in MDP and Euler-Lagrange dynamical equation, we will merely use the term states for MDP and generalized coordinates for the control model, respectively. The connection between $s$ and $(q,\dot q)$ is illustrated by the following remark. 
\begin{remark}\label{remark:equalqa}
The generalized coordinates $(q,\dot q)$ have the capacity to express all the dynamical information of the robot. Hence, the state $s$ in MDP is equivalent to $(q,\dot q)$ in the control theory via some specific transforms. 
\end{remark}
 It means that $s$ and $(q,\dot q)$ can replace each other. As a result, we can investigate $(q,\dot q)$ rather than $s$ in MDP. In robot control, the control inputs are often forces or torques acting on the joints. Under this setting, the control input $u$ is exact the action $a$. Hence, we can use the Euler-Lagrange dynamical equation to analyze the relationship between the generalized coordinates transition $((q_t,\dot q_t),(q_{t+1},\dot q_{t+1}))$ and the control input $u$ instead of analyzing the unknown environment dynamics $T(s'|s,a)$. 
\subsection{Almost Equivalence between LfO and LfD in Deterministic Robot Environments}\label{subsection:eqdeterministic}
In this subsection, we first analyze the inverse dynamics models and the sufficient condition under which the inverse dynamics models of the imitator and the expert are equal everywhere. Then, we give the theorem proving that the inverse dynamics disagreement between LfD and LfO approaches zero, \emph{i.e.}, $\mathbb{D}_{KL}(\rho_{\pi}(a|s,s')||\rho_{E}(a|s,s'))\approx0$, leading to the almost equivalence of both considered approaches. \par
The difference lying between LfD and LfO is the inverse dynamics disagreement, that is,
\begin{align}\label{equation:detailinversedynamicsdis}
\begin{split}
    \mathbb{D}_{KL}(\rho_{\pi}(a|s,s')||\rho_{E}(a|s,s')) = \int_{S\times A\times S} \rho_{\pi}(s,a,s') log \frac{\rho_{\pi}(a|s,s')}{\rho_{E}(a|s,s')} dsdads'.
\end{split}
\end{align}
Eq. \eqref{equation:detailinversedynamicsdis} indicates that the inverse dynamics models of the learner and the expert determine how large the performance gap would be. 
As a result, based on Definition \ref{definition:inversedynamicsmodel} and Eq. \eqref{equation:detailinversedynamicsdis}, if the environment dynamics $T(s'|s,a)$ plusing the imitator policy $\pi_{\theta}(a|s)$ are able to guarantee that $\rho_{\pi}(a|s,s')=\rho_{E}(a|s,s')$ for every $(s,s',a)\in\{S\times S\times A\}$, then,  $\mathbb{D}_{KL}(\rho_{\pi}(a|s,s')||\rho_{E}(a|s,s'))=0$ and no difference exists between the two focus approaches. For deterministic robot environments, the theorem on the almost equivalence of the inverse dynamics models for the learner and the expert is given as follows.
\begin{theorem}\label{theorem1}
Under Assumption \ref{assum:samedynamics}, for deterministic robot controlled plant, the inverse dynamics disagreement between the learner policy and the expert policy approaches zero.
\end{theorem}
\begin{proof}
Our analysis relies on the deterministic property and the Euler-Lagrange dynamical equation. We first use the deterministic character to simplify inverse dynamics models and obtain the sufficient condition to ensure the equality of inverse dynamics models between LfO and LfD. Subsequently, we employ the Euler-Lagrange dynamical equation to prove the uniqueness of the action for a feasible state transition $(s,s')$. \par
\textit{Sufficient Condition.} For a deterministic system, the state-action-state tuples $(s,a,s')$ recorded by interacting with the environment satisfy that $T(s'|s,a)=1$. Assume actions in set $A$ can be divided into two subsets, in which the actions belonging to the first subset can transfer state $s$ to $s'$ while the other could not. To better represent the first part of actions, we use symbol $A_{f}$ in which the subscript $f$ stands for feasibility. Thus, the inverse dynamics models for the imitator and the demonstrator could be simplified as follows:
\begin{align} 
    &\rho_{\pi}(a|s,s')=\frac{\pi_{\theta}(a|s)}{\int_{A_f} \pi_{\theta}(\bar a|s) d\bar a},\label{euqation:inversemodel0}\\
    &\rho_{E}(a|s,s')=\frac{\pi_{E}(a|s)}{\int_{A_f} \pi_{E}(\bar a|s) d\bar a}.\label{euqation:inversemodel1}
\end{align} 
Since that we have no access to the expert policy, the only way to guarantee $\rho_{\pi}(a|s,s')=\rho_{E}(a|s,s')$ is $\left| A_f\right|=1$, meaning that there is only one element in the set $A_f$. In other words, for a feasible state transition $(s,s')$, there is a unique action $a$ which can transfer the current state $s$ to the next $s'$. Thus, the inverse dynamics models are obtained as:
\begin{align*}
    \rho_{\pi}(a|s,s')=\frac{\pi_{\theta}(a|s)}{\pi_{\theta}(a|s)}=1=\frac{\pi_{E}(a|s)}{\pi_{E}(a|s)}=\rho_{E}(a|s,s').
\end{align*}
Now, the equality of $\rho_{\pi}(a|s,s')$ and $\rho_{E}(a|s,s')$ turns into whether the unique existence of action $a$ holds for a feasible state transition $(s,s')$.\par 
\textit{Uniqueness of the Action.} According to Remark \ref{remark:equalqa}, we would focus on whether there exist more than one control inputs $u$ which can transfer the generalized coordinates from $(q_t,\dot q_t)$ to $(q_{t+1},\dot q_{t+1})$ where $t$ is the current timestep. We assume that there exist at least two possible control inputs $u_0$ and $u_1$, which can bring the current state $(q_t,\dot q_t)$ to the next $(q_{t+1},\dot q_{t+1})$ and $u_0 \neq u_1$. When applying $u$ to the robot, accelerations $\ddot q$ are generated from Eq. \eqref{equation:Euler-La}. Based on Euler Methods \cite{butcher2008numerical} and Euler-Lagrange dynamical equation (Lemma \ref{lemma:lagrange}), we prove that:
\begin{align*}
    u_0 \approx u_1.
\end{align*}
This result contradicts the assumption that there exist at least two different control inputs that can achieve the same generalized coordinates transition. Hence, it means that given a system generalized coordinates change, there is only one corresponding control input $u_{uni}$. Consequently, based on Remark \ref{remark:equalqa}, the existence and uniqueness of the action $a$ for a feasible state transition $(s,s')$ have been proved. Returning back to Eqs. \eqref{euqation:inversemodel0} and \eqref{euqation:inversemodel1}, we can guarantee that the inverse dynamics models for the learner and the expert are almost equal everywhere. In conclusion, for deterministic robot tasks, the almost equivalence between LfD and LfO holds. 
\end{proof}
Theorem \ref{theorem1} tells us that LfO is almost equivalent to LfD, and thus, there would be no performance difference in deterministic robot environments. Hence, researchers could safely employ LfO in this kind of environments without worrying about performance degeneration, which greatly expands its application range. We provide a more detailed proof in the Appendix.
\subsection{Almost Equivalence between LfO and LfD in Robot Environments with Bounded Randomness}\label{subsection:random}
In the robot system, Eq. \eqref{equation:Euler-La} is an ideal analytical model of the robot system, which does not contain disturbances \cite{liu2000disturbance}, un-modeled dynamics \cite{nguyen2015adaptive}, and parameter uncertainties \cite{burkan2003upper}. These factors will lead to stochasticity in the environment dynamics and impair the universality of Theorem \ref{theorem1}. To further cope with the situation in reality, we assume bounded randomness in the robot environment dynamics \cite{lu2010experimental,lin2011modeling,li2017robustness}. This means that the environment dynamics $T(s'|s,a)$ can randomly transfer the current state $s$ with action $a$ to a range of next states $\{s'|s'\in(s'_d-\frac{\epsilon}{2},s'_d+\frac{\epsilon}{2})\}$, where $s'_d$ is the certain next state in the deterministic setting and $\epsilon$ is the bound of the randomness which is assumed to be small. This assumption is reasonable and is able to cover the environment dynamics in the real-world. Then, a corollary of the inverse dynamics disagreement between LfD and LfO in random robot environments is presented below. 
\begin{corollary}\label{corollary1}
When the randomness of the environment dynamics is bounded by a small number $\epsilon$ and policies are Lipshitz continuous, the inverse dynamics disagreement between the learner policy and the expert policy approaches zero.
\end{corollary}
This means that even in the robot environment with bounded randomness LfO is almost equivalent to LfD. Corollary \ref{corollary1} extends the Theorem \ref{theorem1} to the stochastic setting, and enlarges the application fields of LfO. The proof of this corollary is in the Appendix. \par
It should be noted that the even the optimizing targets for LfD and LfO are almost same, this does not mean the performance of LfD and LfO could always be the same in the experiments because both LfD and LfO are not guaranteed to achieve the global optima especially with deep neural networks. 
\section{Experiments}
In this section, we compare the experimental results of LfD and LfO on various Mujoco tasks ranging from toy tasks (\emph{e.g.,} InvertedPendulum-v2) to complicated ones (\emph{e.g.,} Humanoid-v2) in the OpenAI Gym \cite{brockman2016openai}. All of these environments satisfy Assumption \ref{assum:samedynamics} and Lemmas \ref{lemma:dynamics}-\ref{lemma:lagrange}, which indicates that they are deterministic robot locomotion tasks. We also provide some conjectures on the ``performance gap'', \emph{i.e.}, the factors that may lead to the performance gap. 

\subsection{Setup}
In this paper, we use OpenAI Mujoco tasks (\emph{e.g.,} InvertedPendulum-v2 and Humanoid-v2) for our experiments, which are continuous deterministic robot locomotion tasks. 
Specifically, the observation and action dimensions for InvertedPendulum-v2 are $4$ and $1$, respectively. Humanoid-v2 has 376-dimensional observation and 17-dimensional action. To generate the expert data, we first train an expert policy using the reinforcement learning (RL) algorithm SAC \cite{haarnoja2018soft}; then we execute the policy deterministically in the environment to collect expert trajectories composed of state transition tuples or state-action pairs. Each trajectory consists of 1000 state transitions or state-action pairs. \par
As for the implementation of GAIfO and GAIL, we follow the implementation of GAIL in the OpenAI Baselines \cite{dhariwal2017baselines} and implement the GAIfO consequently. Following \cite{dhariwal2017baselines}, we use the RL algorithm TRPO \cite{schulman2015trust} to serve as the generator. 
Commonly used hyper-parameters are as follows: the hidden size for all networks is $(100,100)$ and we use $tanh$ as the activation function; the learning rate for the discriminator is $3 \times 10^{-4}$ while that for the value network is $1 \times 10^{-3}$; the ratio between the update of the generator network and discriminator network is set to $3:1$. We only give the hidden layer size of neural networks since the input layer and output layer size varies in different environments. The input size of GAIL is defined by $(s,a)$ whereas that in GAIfO is determined by $(s,s')$, which distinguishes the discriminators in GAIL and GAIfO. More details are in the Appendix. 
\begin{figure*}[htbp]
\centering
    \subfigure[InvertedPendulum-v2]{
    \includegraphics[width=0.3\textwidth]{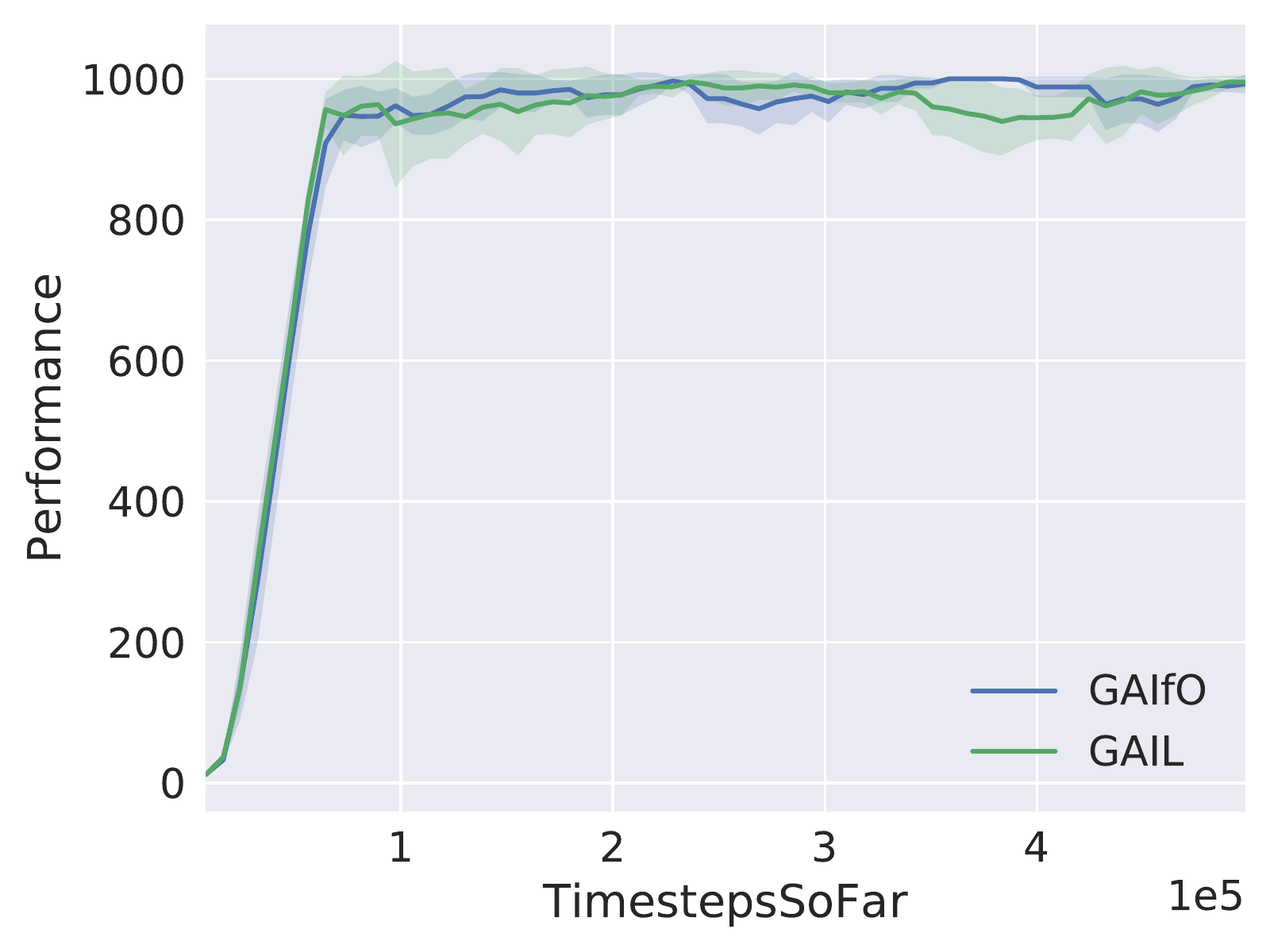}
    }
    \subfigure[Hopper-v2]{
    \includegraphics[width=0.3\textwidth]{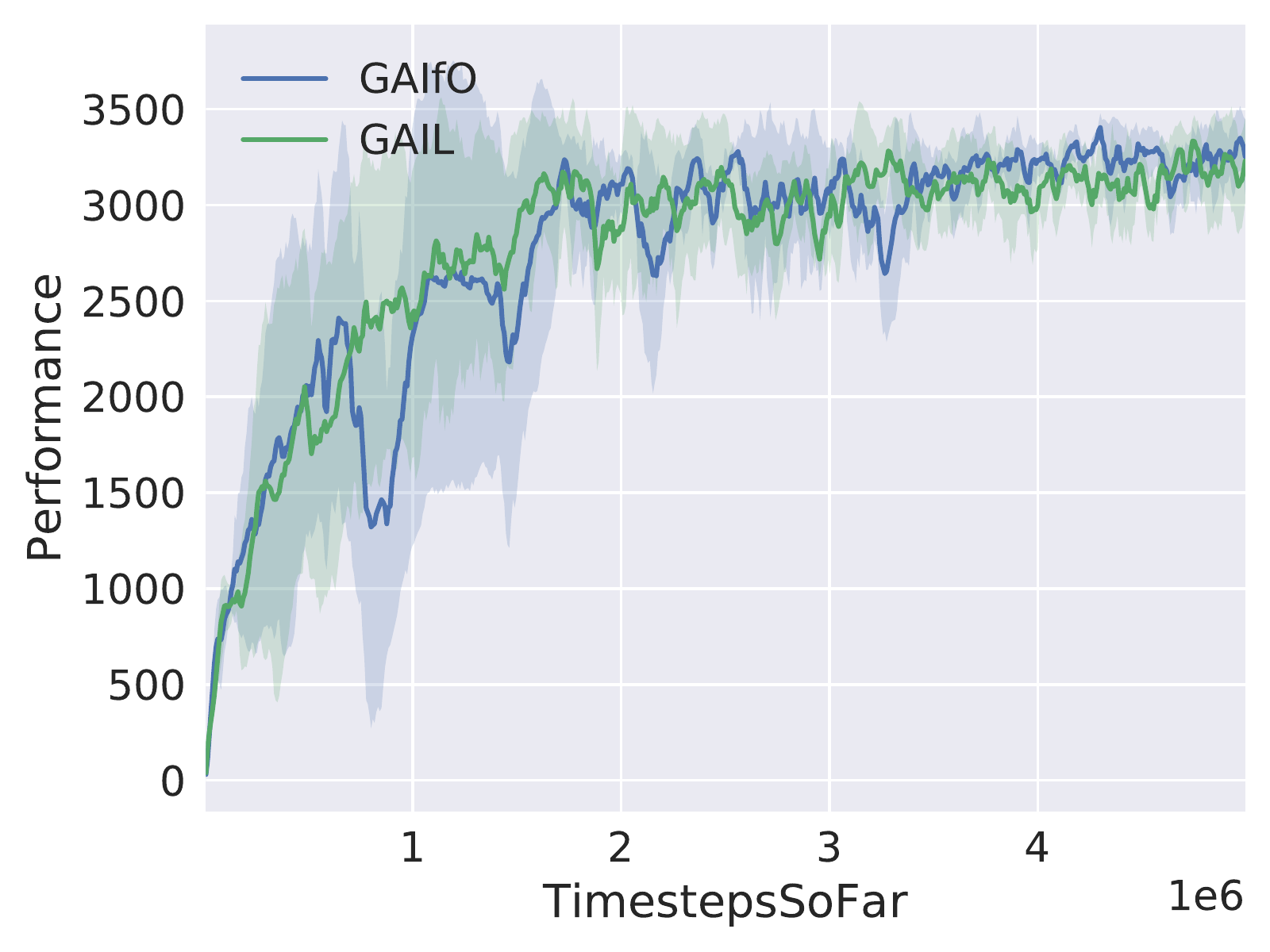}
    }
    \subfigure[Walker2d-v2]{
    \includegraphics[width=0.3\textwidth]{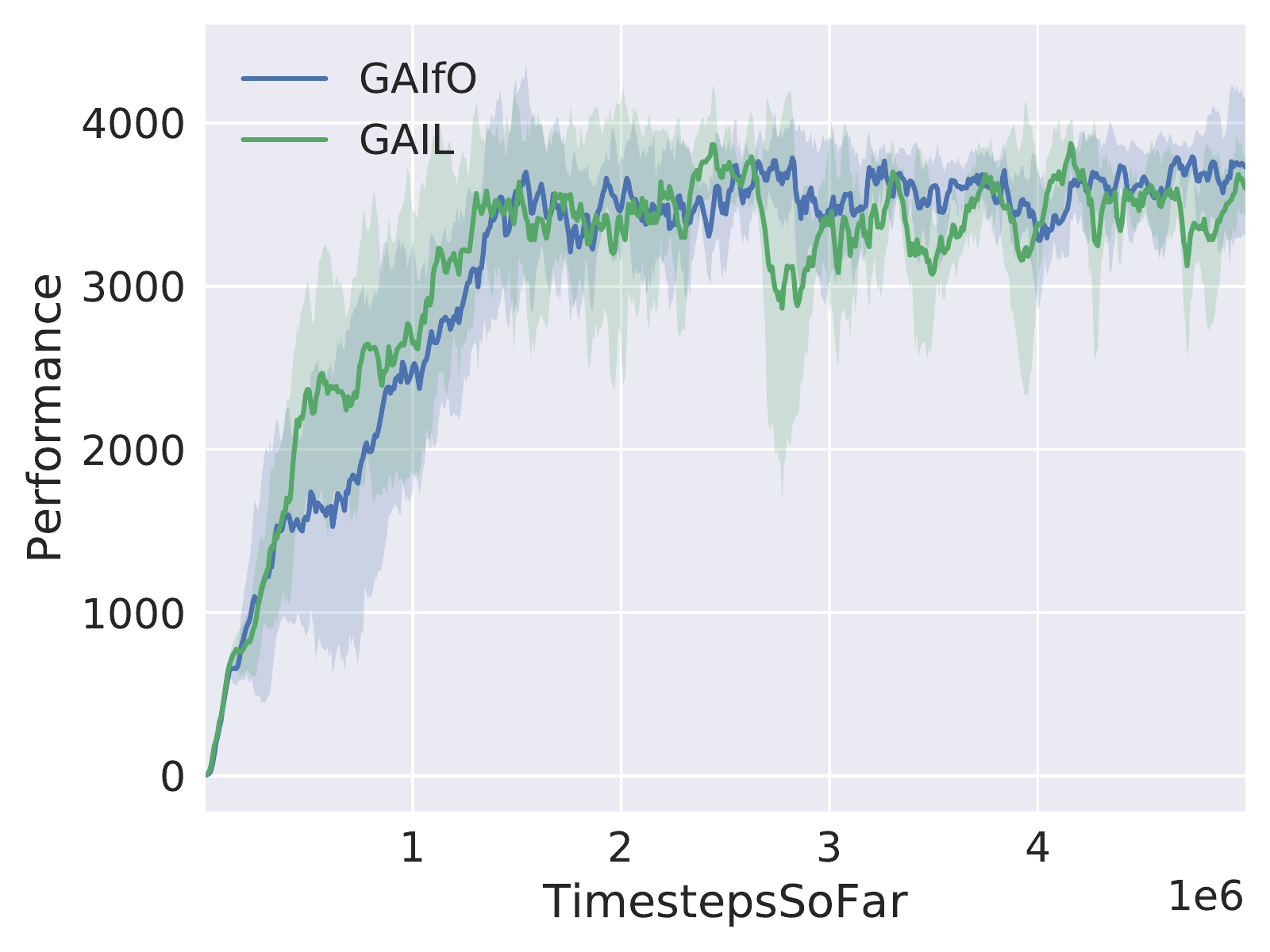}
    }
    \subfigure[HalfCheetah-v2]{
    \includegraphics[width=0.3\textwidth]{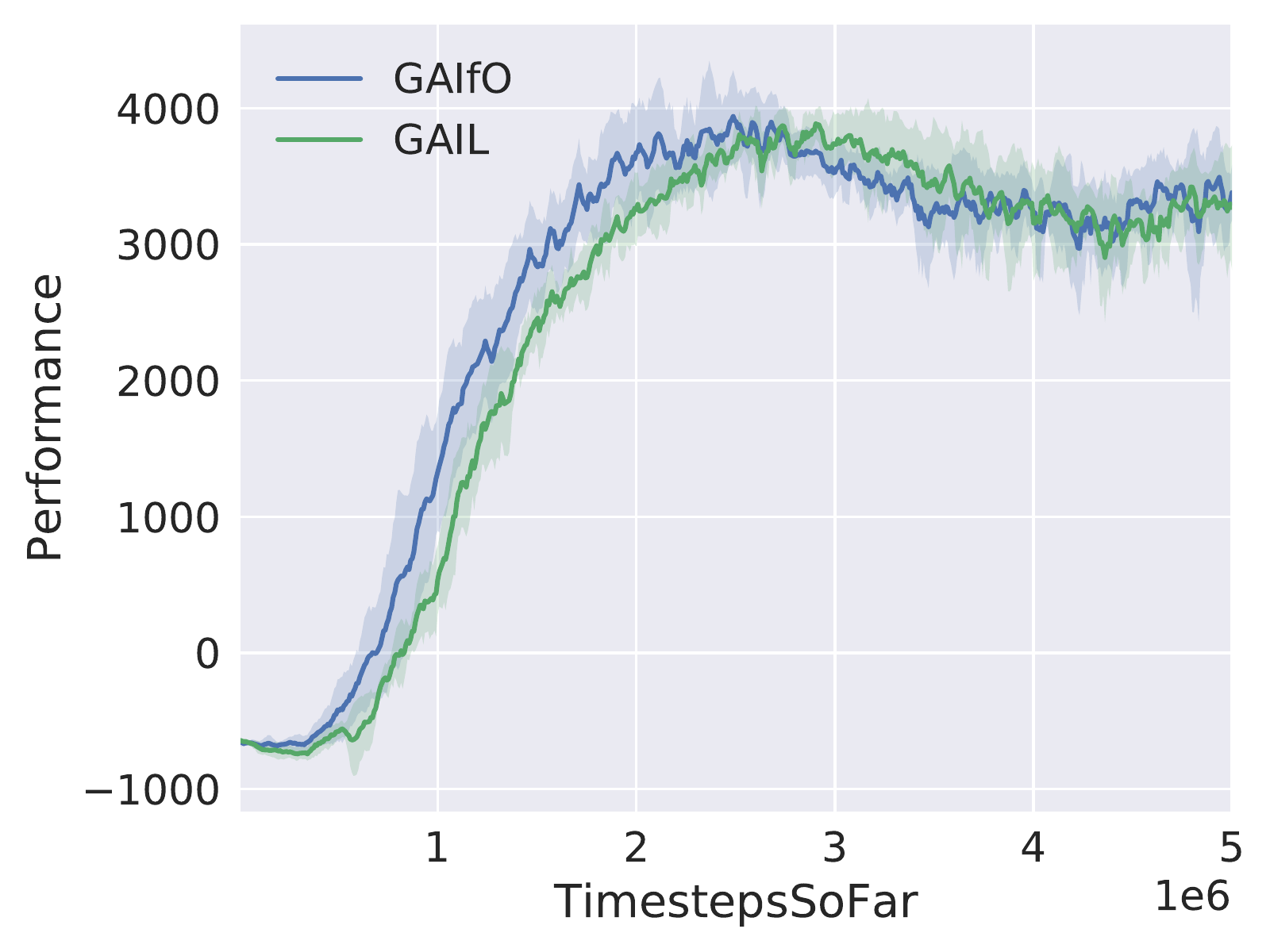}
    }
    \subfigure[Humanoid-v2]{
    \includegraphics[width=0.3\textwidth]{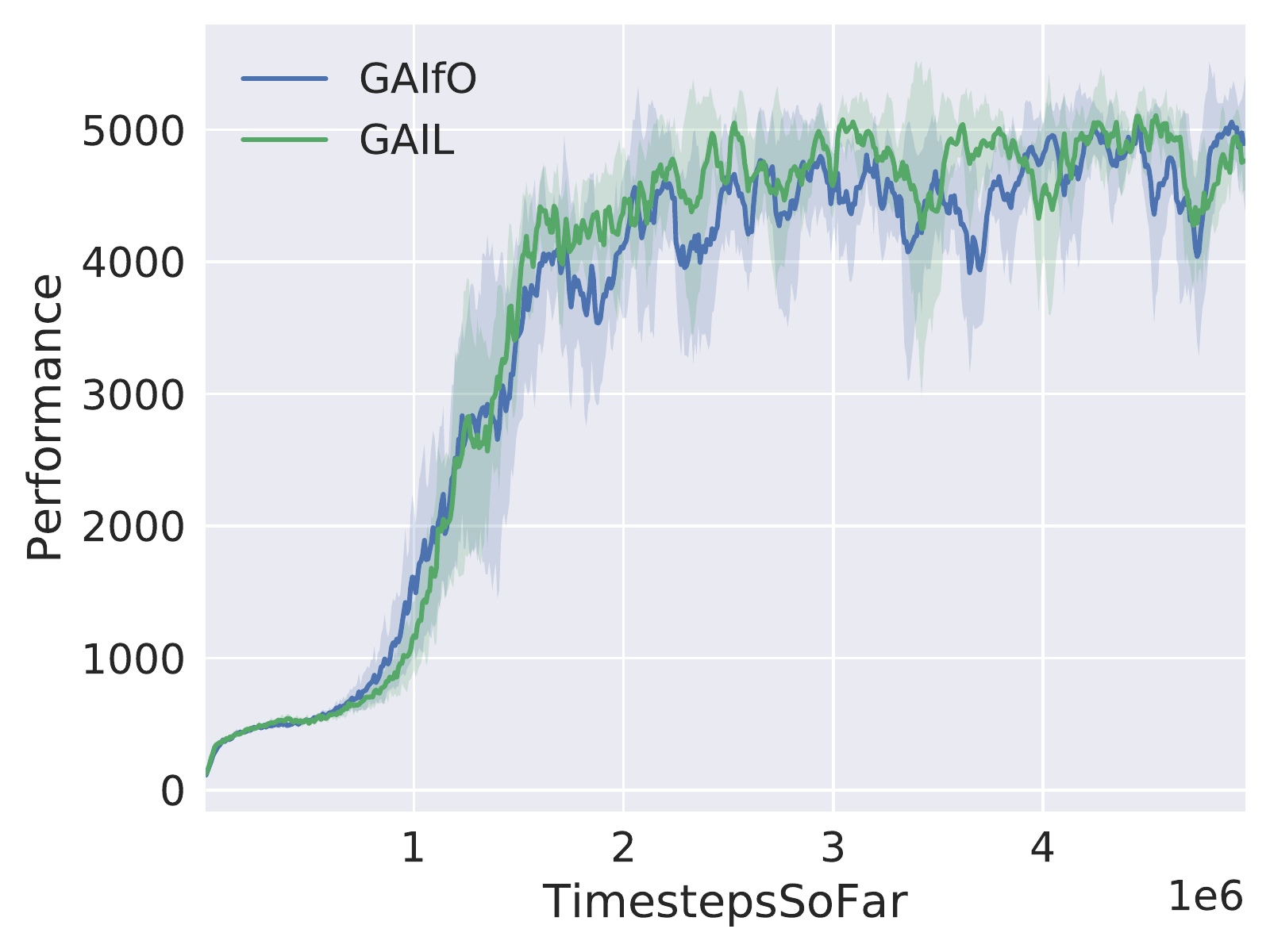}
    }
    \subfigure[HumanoidStandup-v2]{
    \includegraphics[width=0.3\textwidth]{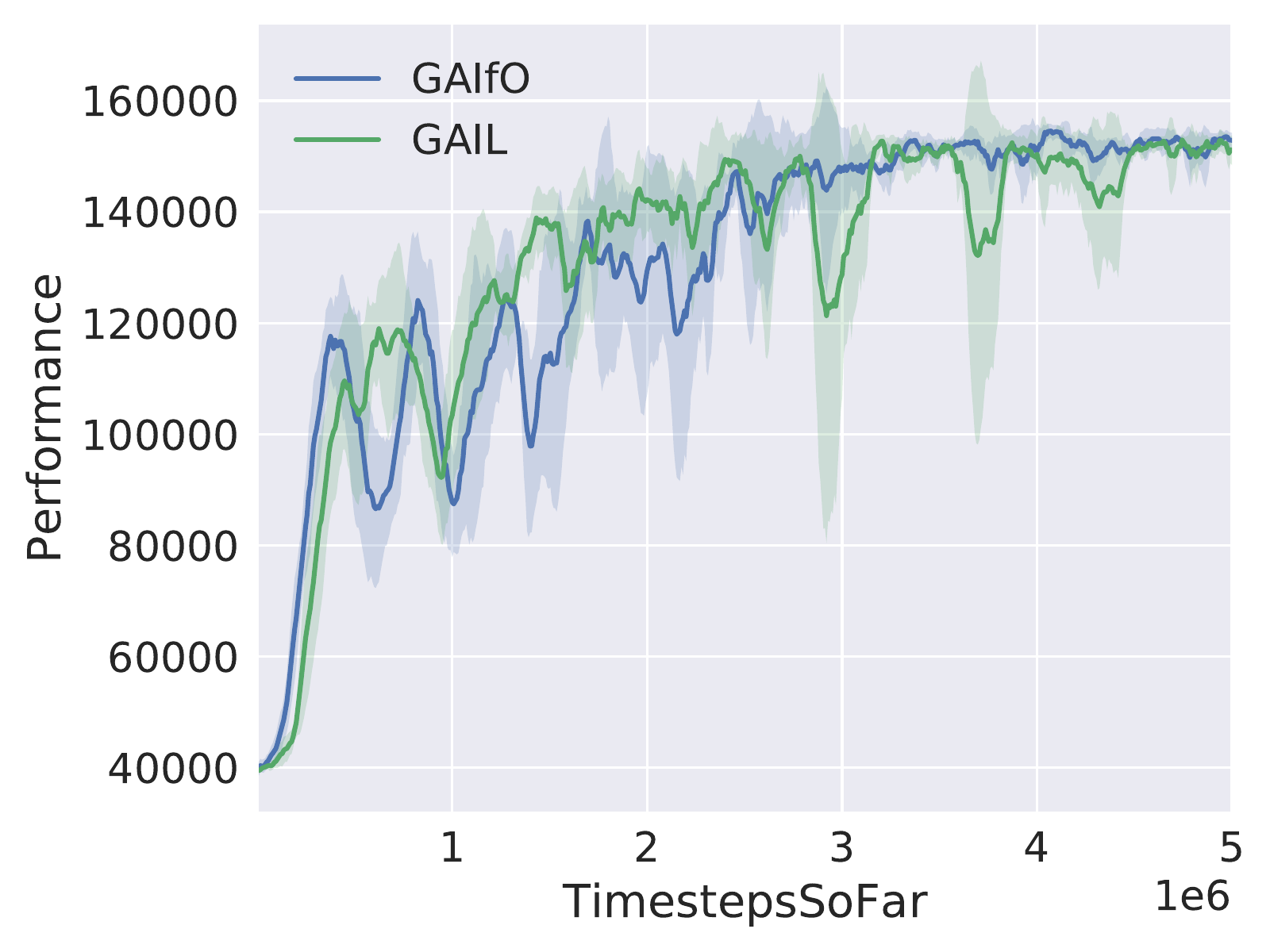}
    }
\caption{Performance of each approach on MuJoCo tasks. Performance is measured with average episode return and x-axis stands for time steps interacting with the environment. Each algorithm is evaluated with 5 random seeds.}
\label{learningcurve}
\end{figure*}

\begin{table*}[h!]
\centering
  \caption{Environments and overall performances of compared algorithms. The numbers before and after the plus or minus symbol are average episode reward and standard error, respectively.}
  \label{table2}
  \centering
  \begin{tabular}{c c c c}
    \hline
    Environment     & Expert     & GAIL & GAIfO\\
    \hline
    InvertedPendulum-v2 & 1000.0$\pm$0.0  & 995.4$\pm$9.2  & 992.7$\pm$13.1   \\
    Hopper-v2     & 3491.8$\pm$37.3  & 3229.6$\pm$221.6 & 3251.6$\pm$195.5\\
    Walker2d-v2     & 3999.6$\pm$10.5  &  3600.9$\pm$201.2 & 3731.3$\pm$424.8\\
    HalfCheetah-v2     & 4993.4$\pm$65.0 & 3270.4$\pm$458.6  & 3381.6$\pm$202.8  \\
    Humanoid-v2     &  5600.4$\pm$10.0    & 4764.8$\pm$227.0 & 4896.9$\pm$507.1\\
    HumanoidStandup-v2   & 155525.8$\pm$2179.9      & 151093.4$\pm$2672.0 & 152898.4$\pm$1333.7\\
    \hline
  \end{tabular}
\end{table*}

\subsection{Results}
In this subsection, we provide the qualitative and quantitative results of GAIL and GAIfO on MuJoCo tasks, which are shown in Fig. \ref{learningcurve} and Tab. \ref{table2}. Each task is trained with 5 random seeds and the episode cumulative reward is employed to evaluate the performance of both methods. All the other settings, including the hyper-parameters, are kept the same for fair comparisons. More details are presented in the Appendix. 
In Fig. \ref{learningcurve}, the mean and the standard deviation of the episode cumulative rewards are illustrated with a solid line and shaded area, respectively, while the numerical results for the expert, GAIL, and GAIfO are listed in Tab. \ref{table2}. From Fig. \ref{learningcurve} and Tab. \ref{table2}, it is clear that (1) both LfD and LfO can achieve the expert-level performance at the end of training; (2) LfO could achieve comparative performance to LfD, \emph{i.e.}, no performance gap between LfO and LfD is noticed. In addition, we adopt a third party open-sourced implementation of GAIL and GAIfO to validate their performance \cite{ota2020tf2rl}, whose learning curves are similar to ours and are presented in the Appendix. \par 
In summary, in deterministic robot tasks, LfO is almost equivalent to LfD, which further confirms our theoretical finding in Theorem \ref{theorem1}.

\subsection{Discussion and Suggestions}
In this subsection, we aim to provide some conjectures and discussions on why previous studies hold the point that there exists a performance gap between LfD and LfO. As far as we are concerned, two factors may affect the performance of LfO and LfD.
\begin{itemize}
  \item [1)] 
  It's well known that current reinforcement learning strategies are comparatively unstable and require elaborate implementation of the algorithms \cite{henderson2017deep,hutson2018artificial, dasagi2019ctrl}. Theorem \ref{theorem1} and Corollary \ref{corollary1} demonstrates that the optimizing targets for LfD and LfO in deterministic robot environments or robot environments with bounded randomness remain almost same, which proves that there would be almost no performance gap. However, in practice, many implementation factors may affect the final performance of these algorithms. Evidently, LfO is more difficult than LfD, since it needs to predict the right action from the expert trajectories. Thus, LfO is more likely to converge at a local optima compared to LfD if some implementation tricks in reinforcement learning or deep learning are not employed. 
  Specifically, the input normalization for the policy and value network plays a very important role among them. To illustrate the importance of input normalization, we further conduct an ablation experiment by training LfO with/without input normalization. As can be seen in Fig. \ref{discussion}, it is clear that the performance is boosted with input normalization. Though this technique helps improve the performance, it does not affect the core of LfO.
 
  \item [2)]
  Another possible reason is that the performance of LfO might be affected by the training instability of GAN \cite{gulrajani2017improved, zhang2019consistency}, since the exploration in LfO is more difficult than that in LfD. Suffering from the training problems of GAN (\emph{e.g.,} vanishing gradients, model collapse, \emph{etc.}), LfO may get stuck at a local optima and perform poorly. To further confirm this point and take full advantage of LfO, we adopt Spectral Normalization \cite{miyato2018spectral}, which is a commonly used technique to stabilize the GAN training. As shown in Fig. \ref{discussion}, after introducing Spectral Normalization, the performance of LfO and LfD are improved. The introduction of spectral normalization does not change the optimizing target and just helps stabilize the training process. 
\end{itemize}\par
In summary, we believe that the performance gap previously reported might be caused by implementation details or training un-stability rather than the inverse dynamics disagreement. Our theoretical analysis and experiments suggest that LfO is almost equivalent to LfD. To further promote research, we will open source our code and dataset to help the reproducibility.
\begin{figure}[htbp]
\centering
    \subfigure[Input Noramlization]{
    \includegraphics[width=0.4\textwidth]{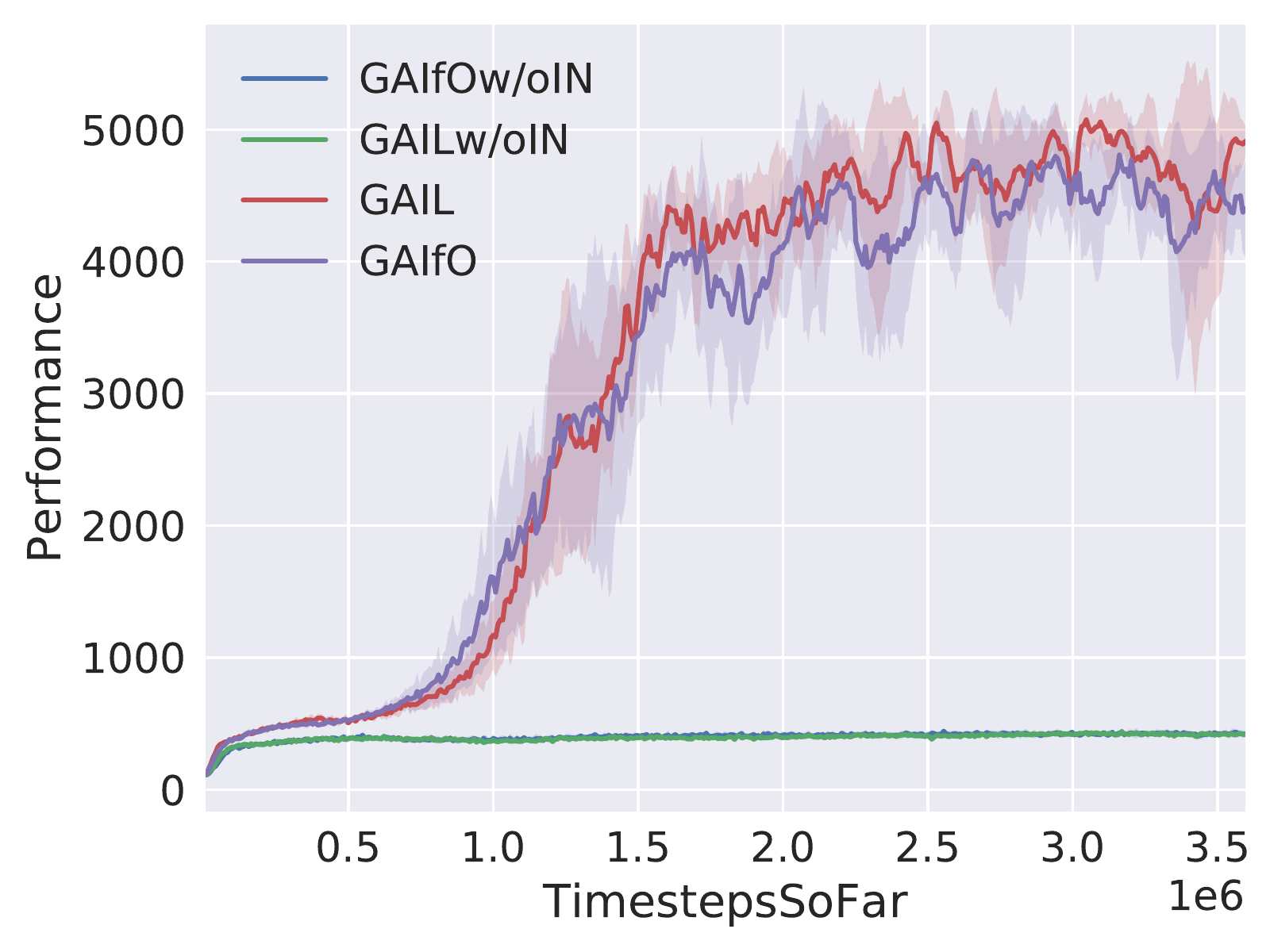}
    }
    \subfigure[Spectral Normalization]{
    \includegraphics[width=0.4\textwidth]{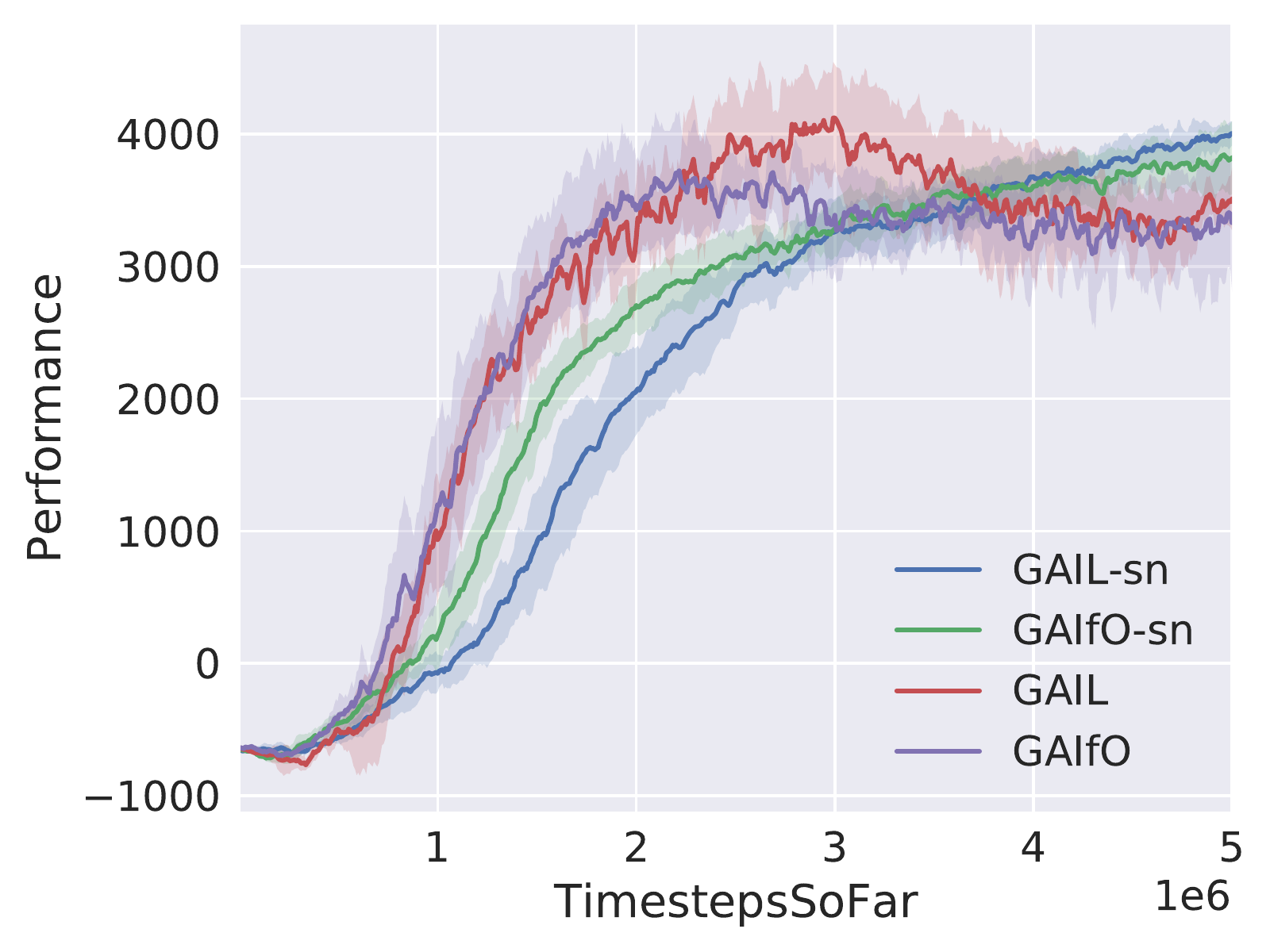}
    }
\caption{Impact of input normalization and spectral normalization on LfO and LfD. The performance measurement criterion is same as described in Fig \ref{learningcurve}. In (a), we test the impact of input normalization on GAIL and GAIfO in Humanoid-v2, in which GAIL and GAIfO uses input normalization in their policy and value networks while GAILw/oIN and GAIfOw/oIN do not; In (b), the learning curves of GAIL and GAIfO against their spectral noramlizaed versions GAIL-sn and GAIfO-sn in HalfCheetah-v2 are presented.}
\label{discussion}
\end{figure}

\section{Conclusions}
In this paper, we delve into the difference between LfD and LfO specifically from two algorithms, GAIL and GAIfO, which are the two most commonly used imitation learning algorithms. From the perspective of control theory and inverse dynamics model, we prove that, for deterministic robot systems, the performance gap between LfO and LfD approaches zero, \emph{i.e.}, the almost equivalence between LfO and LfD is proved. To further relax the deterministic constraint and better adapt to the environment in practice, we consider the bounded randomness in the robot environment and prove that the optimizing targets for LfD and LfO remain almost same in a more generalized setting, which paves the way for the application of LfO in real-world. In other words, LfO is almost equivalent to LfD in deterministic robot environments or robot environments with bounded randomness. Extensive experiments on OpenAI MuJoCo tasks are conducted and empirically demonstrate that LfO achieves comparable performance to LfD.

Based on our studies, we provide the state-of-the-art implementation of GAIfO with the code and dataset open-sourced, which are beneficial for the imitation learning community. Our analysis gives a deep insight into imitation learning from observation and will for certain promote the applicability of LfO in real-world robot tasks. Future work would focus on the analysis of the difficulty for finding the right action with given state transitions for LfO.
\section{Ethics}
Like most imitation learning algorithms, generative adversarial imitation learning discussed in our paper has the potential to be applied in industrial production and robots for daily life. But applying the algorithms in the researches of artificial intelligence to real-world applications is not an easy task. Our work extends the possible application areas of imitation learning, and to some extent, our theoretical results mitigate the obstacles for employing imitation learning methods in reality. Since these learning methods do not require humans to intervene, if they are successfully applied in practice fewer workers will be needed. For example, the stable walking of Atlas robot in Boston Dynamics is achieved with the efforts from many engineers who work on control, planning, and perception. If imitation learning algorithms could be employed to teach Atlas to walk like humans, Boston Dynamics would not need to hire so many engineers. However, on the other hand, imitation learning can help increase industrial productivity and reduce production costs simultaneously, which will create large social values that most people can benefit from. 


\bibliographystyle{ieee_fullname}
\bibliography{egbib}

\clearpage
\onecolumn
\setcounter{theorem}{0}
\setcounter{corollary}{0}
\section*{Appendix——On the Guaranteed Almost Equivalence between Imitation Learning from Observation and Demonstration}
In the appendix, we provide some supplementary material which supports our claims that the almost equivalence between LfO and LfD holds in robot environments under deterministic model and bounded randomness both theoretically and empirically. The contents in the appendix can be divided into two parts. In the first part, we present the detailed proof process of Theorem \ref{theorem1} and Corollary \ref{corollary1} to demonstrate that the almost equivalence is guaranteed in robot environments under deterministic model and bounded randomness, respectively. In the other part, we further validate our theoretical analysis by providing additional experimental results, including how to generate the expert data and conduct some additional experiments to demonstrate that LfO can perform as well as LfD.   
\section*{A. Proof}
In this section, we provide detailed proof of Theorem \ref{theorem1} and Corollary \ref{corollary1}. 
\subsection*{A.1 Theorem 1 and Proof}
\begin{theorem}
Under Assumption \ref{assum:samedynamics}, for deterministic robot controlled plant, the inverse dynamics disagreement between the learner policy and the expert policy approaches zero.
\end{theorem}
\begin{proof}
Our analysis relies on the deterministic property and the Euler-Lagrange dynamical equation. We ﬁrst use the deterministic character to simplify inverse dynamics models
and obtain the sufﬁcient condition to ensure the equality of inverse dynamics models between LfO and LfD. Subsequently, we employ the Euler-Lagrange dynamical equation to prove the uniqueness of the action for a feasible state transition $(s,s')$. \par

\textit{Sufﬁcient Condition}. For a deterministic system, the state-action-state tuples $(s,a,s')$ recorded by interacting with the environment satisfy that $T(s'|s,a)=1$. The environment dynamics $T(s'|s,a)$ is one of the factors which may lead to a mismatch of inverse dynamics models. Assume that actions in set $A$ can be divided into two subsets, in which the actions belonging to the first subset can transfer state $s$ to $s'$ while the other could not. To better represent the first action subset, we use the symbol $A_{f}$ in which the subscript $f$ stands for feasibility. Thus, the inverse dynamics models for the imitator and the demonstrator could be simplified as follows:
\begin{align} 
    &\rho_{\pi}(a|s,s')=\frac{\pi_{\theta}(a|s)}{\int_{A_f} \pi_{\theta}(\bar a|s) d\bar a},\label{euqation:app_inversemodel0}\\
    &\rho_{E}(a|s,s')=\frac{\pi_{E}(a|s)}{\int_{A_f} \pi_{E}(\bar a|s) d\bar a}.\label{euqation:app_inversemodel1}
\end{align}
In practice, we have no access to the expert policy but just some state action samples recorded by executing the expert policy. In other words, the policy of the expert is unknown and unavailable. Therefore, the probability of taking actions using the expert policy is undiscovered. As a result, it is not possible for us to ensure the equality of $\rho_{\pi}(a|s,s')$ and $\rho_{E}(a|s,s')$ by controlling the action taking probabilities of the learner policy. The only way to guarantee $\rho_{\pi}(a|s,s')=\rho_{E}(a|s,s')$ is that $\left| A_f\right|=1$, meaning that there is only one element in the set $A_f$. In other words, for a feasible state transition $(s,s')$, there is a unique action $a$ which can transfer the current state $s$ to the next $s'$ by executing this action. Thus, the inverse dynamics models are obtained as:
\begin{align*}
    \rho_{\pi}(a|s,s')=\frac{\pi_{\theta}(a|s)}{\pi_{\theta}(a|s)}=1=\frac{\pi_{E}(a|s)}{\pi_{E}(a|s)}=\rho_{E}(a|s,s').
\end{align*}
It should be noted that, both policies are in the Gaussian form, so the probability of taking any action is not zero, which avoids singularities in the above equation. \par
Now, the equality of $\rho_{\pi}(a|s,s')$ and $\rho_{E}(a|s,s')$ turns into whether the unique existence of action $a$ holds for a feasible state transition $(s,s')$. We believe that this condition is true in a deterministic robot environment.  \par 

\textit{Uniqueness of the Action}. According to Remark \ref{remark:equalqa}, we would focus on whether there exist more than one control inputs $u$ that can transfer the generalized coordinates from $(q_t,\dot q_t)$ to $(q_{t+1},\dot q_{t+1})$ given the Euler-Lagrange dynamical equation, where $t$ is the current timestep. We assume that there exist at least two possible control inputs $u_0$ and $u_1$ that can bring the current state $(q_t,\dot q_t)$ to the next $(q_{t+1},\dot q_{t+1})$, and $u_0 \neq u_1$. When applying $u$ to the robot, accelerations $\ddot q$ are generated with Eq. \eqref{equation:Euler-La}. According to the initial speed value $\dot q_t$ and by integrating the acceleration, the increment for speed can be determined. Similarly, the increment for the position could also be obtained \cite{butcher2008numerical}:
\begin{align}
    q_{t+1}=q_{t}+\int_{t}^{t+1} \dot q dt\approx q_{t}+\Delta_t \times \dot q \label{equation:evolution0},\\
    \dot q_{t+1}=\dot q_{t}+\int_{t}^{t+1} \ddot q dt\approx \dot q_{t}+\Delta_t \times \ddot q, \label{equation:evolution1}
\end{align}
where $\Delta_t$ is the time interval between timestep $t+1$ and $t$. Eqs. \eqref{equation:evolution0} and \eqref{equation:evolution1} identify the variation process of the generalized coordinates. For the robot dynamics both in simulation and in real applications, the control frequency of the controller remains relatively high, indicating that the time interval is relatively small \cite{hwangbo2019learning,peng2020learning,apgar2018fast}. Thus, we could assume that during the integral time interval, neither $\dot q$ nor $\ddot q$ would change. Therefore, for a given generalized coordinates change, a unique $\ddot q_{uni}$ can be calculated. Combining with the assumption that both $u_0$ and $u_1$ could produce the same state transition, the control inputs $u_0$ and $u_1$ can then generate the same acceleration ${\ddot q}_{uni}$ with the current state $(q_t,\dot q_t)$. 
According to the definition of the Euler-Lagrange dynamical equation, the inertia matrix $M(q)$ is positive definite \cite{westervelt1st,nanos2012cartesian}, meaning that $M(q)$ is invertible. Hence, by transforming Eqs. \eqref{equation:Euler-La} with control input $u_0$ and $u_1$, the following is presented:
\begin{align}
    \ddot q_{uni} \approx M^{-1}(q_t)u_i- M^{-1}(q_t)C(q_t,\dot q_t)\dot q_t -M^{-1}(q_t) G(q_t) \label{equation:generateuniq_inverse0},
\end{align}
where $i\in \{1,2\}$. By utilizing the uniqueness of $\ddot q$, subtracting Eq. \eqref{equation:generateuniq_inverse0} with index $i=0,1$, we obtain,
\begin{align}\label{equation:equalu}
\begin{split}
    0& \approx M^{-1}(q_t)u_0 - M^{-1}(q_t)u_1\\
    & \approx M^{-1}(q_t)(u_0 - u_1).
\end{split}
\end{align}
Since $M(q_t)$ is invertible, $M^{-1}(q_t)$ is also invertible \cite{westervelt1st,nanos2012cartesian}. Thus, Eq. \eqref{equation:equalu} could be further simplified as:
\begin{align}
    u_0 \approx u_1.
\end{align}
This result contradicts the assumption that there exist at least two different control inputs that can achieve the same state transition. Hence, it means that given a system generalized coordinate change, there is only one corresponding control input $u_{uni}$. Consequently, based on Remark \ref{remark:equalqa}, the existence and uniqueness of action $a$ for a feasible state transition $(s,s')$ have been proved. Returning back to Eqs. \eqref{euqation:app_inversemodel0} and \eqref{euqation:app_inversemodel1}, we can guarantee that the inverse dynamics models for the learner and the expert equal everywhere. In conclusion, for deterministic robot tasks, the almost equivalence between LfD and LfO holds.
\end{proof}
\subsection*{A.2 Corollary 1 and Proof}
\begin{corollary}
When the randomness of the environment dynamics is bounded by a small number $\epsilon$  and policies are
Lipshitz continuous, the inverse dynamics disagreement between the learner policy and the expert policy approaches zero.
\end{corollary}
\begin{proof}
For a given state transition $(s,s')$, there should be a small range of actions $A_f=\{a|a\in(a-\frac{\delta}{2},a+\frac{\delta}{2})\}$ which may transit $s$ to $s'$. The parameter $\delta$ is related to $\epsilon$ and assumed to be relatively small. Since that the environment dynamics are not deterministic and the existence of a unique action $a$ for a state transition $(s,s')$ does not hold, we could not simplify the inverse dynamics models as follows: 
\begin{align} 
    &\rho_{\pi}(a|s,s')=\frac{\pi_{\theta}(a|s)}{\int_{A_f} \pi_{\theta}(\bar a|s) d\bar a},\label{euqation:inversemodel00}\\
    &\rho_{E}(a|s,s')=\frac{\pi_{E}(a|s)}{\int_{A_f} \pi_{E}(\bar a|s) d\bar a}.\label{euqation:inversemodel11}
\end{align}
Rewriting the complete inverse dynamics model for the learner policy as follows:
\begin{align}\label{equation:fulloccpmeasure}
\begin{split}
      \rho_{\pi}(a|s,s')=\frac{T (s'|s,a)\pi_{\theta}(a|s)}{\int_{A_f} T(s'|s,\bar a)\pi_{\theta}(\bar a|s) d\bar a}
      =\frac{T (s'|s,a)\pi_{\theta}(a|s)}{\int_{a-\frac{\delta}{2}}^{a+\frac{\delta}{2}} T(s'|s,\bar a)\pi_{\theta}(\bar a|s) d\bar a}.
\end{split}      
\end{align}
It is clear from Eq. \eqref{equation:fulloccpmeasure} that, with the existence of the randomness, the inverse dynamics model involves both the environment dynamics and the probability of the decision policy, making it difficult to derive the relationship between $\rho_{\pi}(a|s,s')$ and $\rho_{E}(a|s,s')$. In addition, it is hard to eliminate the probability of taking a given action $a$ in the denominator and numerator of the inverse dynamics model. To tackle this, we further assume that the policy is continuous and differentiable. Therefore, Eq. \eqref{equation:fulloccpmeasure} could be simplified by extracting the probability of the policy using the Mean Value Theorem \cite{rudin1964principles}:  
\begin{align}\label{equation:pioccuwithXi}
\begin{split}
      \rho_{\pi}(a|s,s')
      &=\frac{\pi_{\theta}(a|s) T (s'|s,a)}{\pi_{\theta}(\widetilde a|s) \int_{a-\frac{\delta}{2}}^{a+\frac{\delta}{2}} T(s'|s,\bar a) d\bar a}\\
      &=\frac{\pi_{\theta}(a|s)}{\pi_{\theta}(\widetilde a|s)}\frac{ T (s'|s,a)}{\int_{a-\frac{\delta}{2}}^{a+\frac{\delta}{2}} T(s'|s,\bar a) d\bar a}\\
      &=\Xi_{\pi} \frac{ T (s'|s,a)}{\int_{a-\frac{\delta}{2}}^{a+\frac{\delta}{2}} T(s'|s,\bar a) d\bar a},
\end{split}
\end{align}
where $\widetilde a \in A_f$ and $\Xi_{\pi}=\pi_{\theta}(a|s)/\pi_{\theta}(\widetilde a|s)$. Assume that the Lipschitz constant for the policy is $L_{\theta}$ such that,
\begin{align}
      \| \pi_{\theta}(a|s)-\pi_{\theta}(\widetilde a|s) \| \leq L_{\theta} \|a - \widetilde a \| \leq \frac{1}{2} L_{\theta} \abs{\delta}.
\end{align}
Accordingly, the inequality for $\pi_{\theta}(a|s)$ and $\pi_{\theta}(\widetilde a|s)$ holds,
\begin{align}
     \pi_{\theta}(a|s) - \frac{1}{2} L_{\theta} \abs{\delta} \leq \pi_{\theta}(\widetilde a|s) \leq \pi_{\theta}(a|s) + \frac{1}{2} L_{\theta} \abs{\delta}.
\end{align}
Thus, we can get the bound of $\Xi_{\pi}$:
\begin{align}
     \frac{\pi_{\theta}(a|s)}{\pi_{\theta}(a|s) + \frac{1}{2} L_{\theta} \abs{\delta}} \leq \Xi_{\pi} \leq \frac{\pi_{\theta}(a|s)}{\pi_{\theta}(a|s) - \frac{1}{2} L_{\theta} \abs{\delta}}.
\end{align}
Since that $\delta$ is relatively small, $L_{\theta} \abs{\delta}$ approaches zero,  $\Xi_{\pi}$ approaches one. Consequently, the inverse dynamics model would essentially be determined by the forward dynamics model $T(s'|s,a)$. Since both the learner and the expert share the same environment dynamics, we can obtain,
\begin{align}\label{equation:EoccuwithXi}
\begin{split}
      \rho_{E}(a|s,s')
      =\Xi_{E} \frac{ T (s'|s,a)}{\int_{a-\frac{\delta}{2}}^{a+\frac{\delta}{2}} T(s'|s,\bar a) d\bar a},
\end{split}
\end{align}
where $\Xi_{E}=\pi_{E}(a|s)/\pi_{E}(\widetilde a|s)$. Replacing Eqs. \eqref{equation:pioccuwithXi} and \eqref{equation:EoccuwithXi} into the inverse dynamics disagreement, we obtain,
\begin{align}
\begin{split}
    \mathbb{D}_{KL}(\rho_{\pi}(a|s,s')||\rho_{E}(a|s,s')) &= \int_{S\times A\times S} \rho_{\pi}(s,a,s') log(\frac{\Xi_{\pi}}{\Xi_{E}}) dsdads'
    \\&\approx \int_{S\times A\times S} \rho_{\pi}(s,a,s') log(1) dsdads'
    \\&\approx 0,
\end{split}
\end{align}
which means that there would be almost no difference in terms of the optimizing target between LfD and LfO even with bounded randomness. Hence, we can conclude that LfO is almost equivalent to LfD in bounded randomness robot environments.
\end{proof}

\section*{B. Environment Specifications and Expert Data}
In this section, we introduce the OpenAI Mujoco tasks used in our experiments and give details on how to generate expert data. 
\subsection*{B.1 Environment Specifications}
The specifications of the tested environments are listed in Table.  \ref{mujocotask}. 
\begin{table*}[h!]
\centering
  \caption{Specifications of OpenAI Mujoco robot tasks.}
  \label{mujocotask}
  \centering
  \begin{tabular}{c c c c}
    \hline
    Environment     & State Space     & Action Space & Max-Step\\
    \hline
    InvertedPendulum-v2 & 4  & 1  & 1000\\
    InvertedDoublePendulum-v2     & 11 & 1  &  1000\\
    Hopper-v2     & 11  & 3 &  1000\\
    Walker2d-v2     & 17  &  6 &  1000\\
    HalfCheetah-v2     & 17  &  6 &  1000 \\
    Humanoid-v2     &  376   & 17 &  1000\\
    HumanoidStandup-v2  &  376   & 17 &  1000\\
    \hline
  \end{tabular}
\end{table*}
\subsection*{B.2 Expert Data}
Expert data is the foundation for imitation learning. In terms of LfD and LfO, the two approaches require different expert data lying in that the former needs both states and actions while the latter merely demands states for accomplishing the task. We use the reinforcement learning algorithm Soft Actor-Critic (SAC) to train an expert policy for the Mujoco tasks \cite{haarnoja2018soft}. Our experts are trained using the implementation of OpenAI Spiningup with the default configuration \cite{SpinningUp2018}. 
Once we have a converged SAC policy, we execute its deterministic part, i.e., the predicted mean action values of each action dimension, in the environment to collect transition tuples including states and actions. The episode length is set to $1000$ and an episode of states and actions is called a trajectory. In terms of LfO, the expert data can be generated by omitting actions in the trajectories and stacking adjacent states to gather $(s,s')$ pairs. Generally, the generation of the expert reference is not crucial for the evaluation of different imitation learning methods, but we incorporate this part to make our work easy to be reproduced and self-contained. \par

\section*{C. Implementation Details}
Our implementation is based on the GAIL in OpenAI Baselines \cite{dhariwal2017baselines}. In this section, the complete hyper-parameters and the number of expert data we used for every task are given. 
Table. \ref{hyperparam} shows all the hyper-parameters for our experiments while Table. \ref{numoftransitions} presents the number of expert trajectories to conduct imitation learning in the paper. 
\begin{table}[!htb]
\caption{Hyper-parameters in experiments.}
\centering
\begin{tabular}{ l|l}
 \hline
 Hyper-parameters& Value\\
 \hline
 Common parameters & \quad\\
 \qquad Network size (For all networks)   & (100,100)  \\
 \qquad Activation & $ReLU$\\
 \qquad Batch size  & 1,024 \\
 \qquad Optimizer  & Adam \\
 \qquad Generator network update times & 3 \\
 \qquad Discriminator network update times & 1 \\
 \hline
 TRPO parameters & \quad\\
 \qquad $\gamma$ (Generalized Advantage Estimation Gamma) & 0.995\\
 \qquad $\lambda$ (Generalized Advantage Estimation Lambda) & 0.97\\
 \qquad Maximum KL & 0.01\\
 \qquad Learning rate (Value network)  & $1 \times 10^{-3}$ \\
 \qquad Value iteration  & 5 \\
 \qquad Policy entropy  & 0.0 \\
 \hline
 Discriminator parameters & \quad\\
 \qquad Learning rate (Discriminator network) & $3 \times 10^{-4}$ \\
 \qquad Discriminator entropy  & $1 \times 10^{-3}$\\
 \hline
\end{tabular}
\label{hyperparam}
\end{table}

\begin{table}[h!]
\centering
  \caption{Number of expert trajectories.}
  \label{numoftransitions}
  \centering
  \begin{tabular}{c c}
    \hline
    Environment     & Expert Trajectories \\
    \hline
    InvertedPendulum-v2 & 20 \\
    InvertedDoublePendulum-v2     & 20 \\
    Hopper-v2     & 20\\
    Walker2d-v2     & 100\\
    HalfCheetah-v2     & 100 \\
    Humanoid-v2     &  20\\
    HumanoidStandup-v2  &  20 \\
    \hline
  \end{tabular}
\end{table}
%

\section*{D. Additional Experiments}
To thoroughly validate our theoretical analysis, we conduct extensive comparison experiments between LfO and LfD on OpenAI Mujoco tasks \cite{brockman2016openai}. Due to the space limit of the paper, some experiment results are presented here.  
\subsection*{D.1 Additional Tested Environments}
The learning curve of InvertedDoublePendulum-v2 is illustrated in Fig. \ref{fig:learningcurveIDP}. The final performance of the expert, GAIL and GAIfO are, 9357.5$\pm$0.1, 8763.4$\pm$ 830.1, and 8832.6$\pm$356.4, respectively. Hence, we can conclude that the performance of GAIfO is comparable to that of GAIL in the environment InvertedDoublePendulum-v2, which supports the almost equivalence claim between LfO and LfD in the paper. 
\begin{figure}[htbp]
\centering
\includegraphics[width=0.4\textwidth]{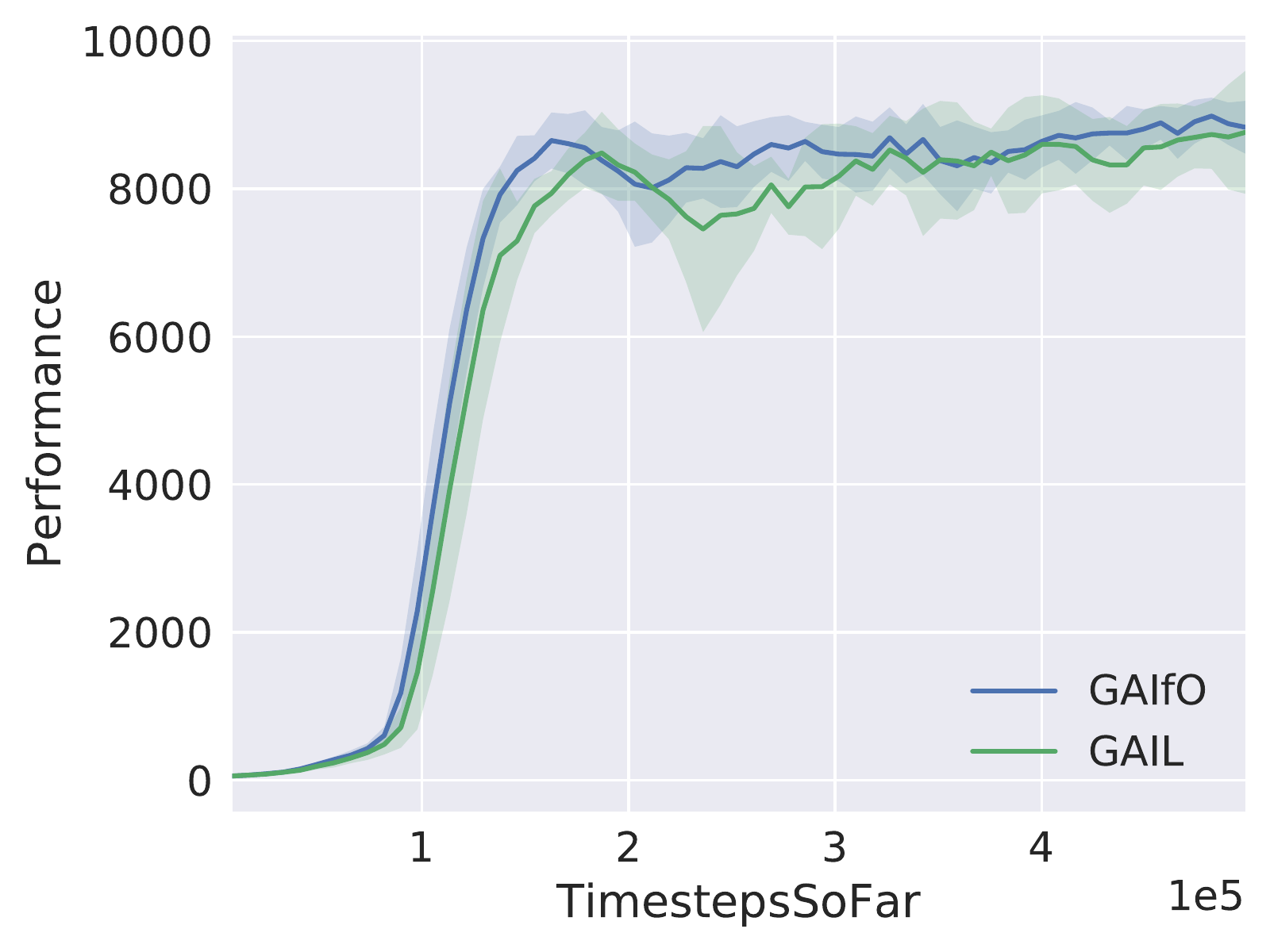}
\caption{Learning curves of GAIL and GAIfO in InvertedDoublePendulum-v2.}
\label{fig:learningcurveIDP}
\end{figure}
\subsection*{D.2 Impact of Trajectory Size}
Actually, GAIL can learn to mimic the expert with several expert trajectories, which means that it has high data utilization efficiency. We test the impact of the number of trajectories used in GAIL and GAIfO in this subsection. From Fig. \ref{impactoftraj}, it is clear that GAIfO is also able to clone expert behaviours with several expert trajectories, which means that GAIfO also has very high data utilization efficiency. Furthermore, the performance of GAIfO is comparable to that of GAIL in spite of the number of expert trajectories. 
\begin{figure*}[htbp]
\centering
    \subfigure[InvertedPendulum-v2]{
    \includegraphics[width=0.3\textwidth]{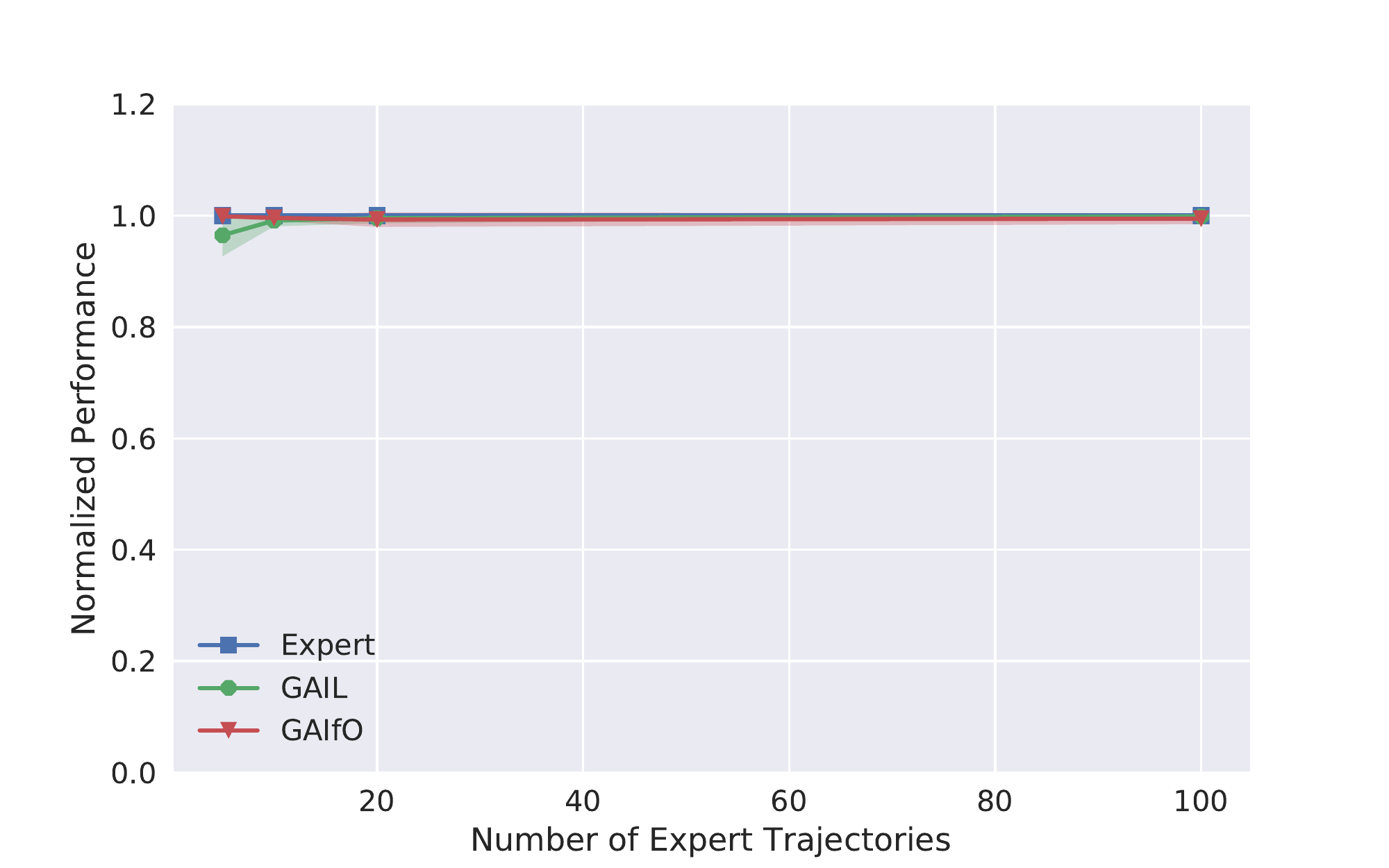}
    }
    \subfigure[InvertedDoublePendulum-v2]{
    \includegraphics[width=0.3\textwidth]{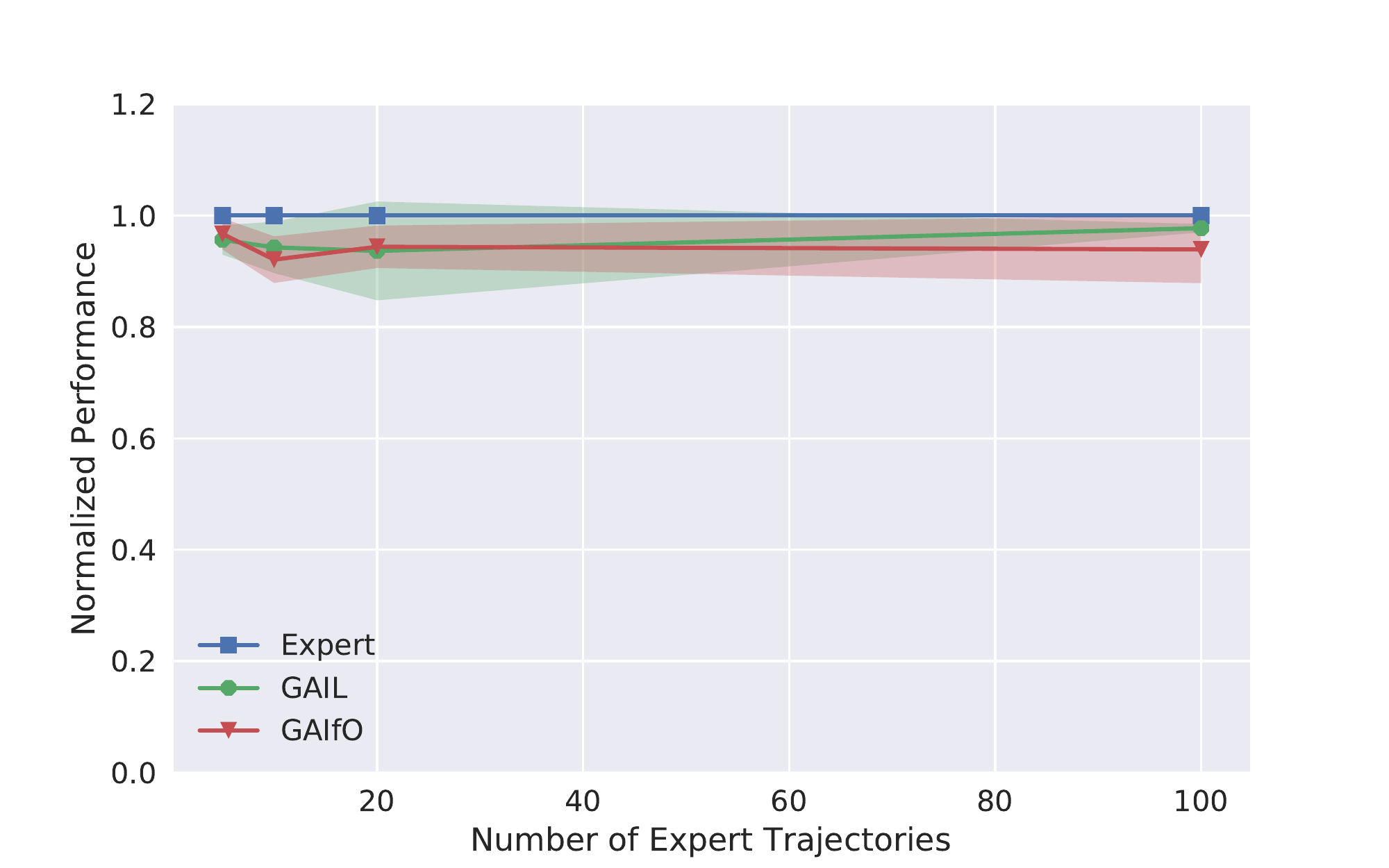}
    }
    \subfigure[Hopper-v2]{
    \includegraphics[width=0.3\textwidth]{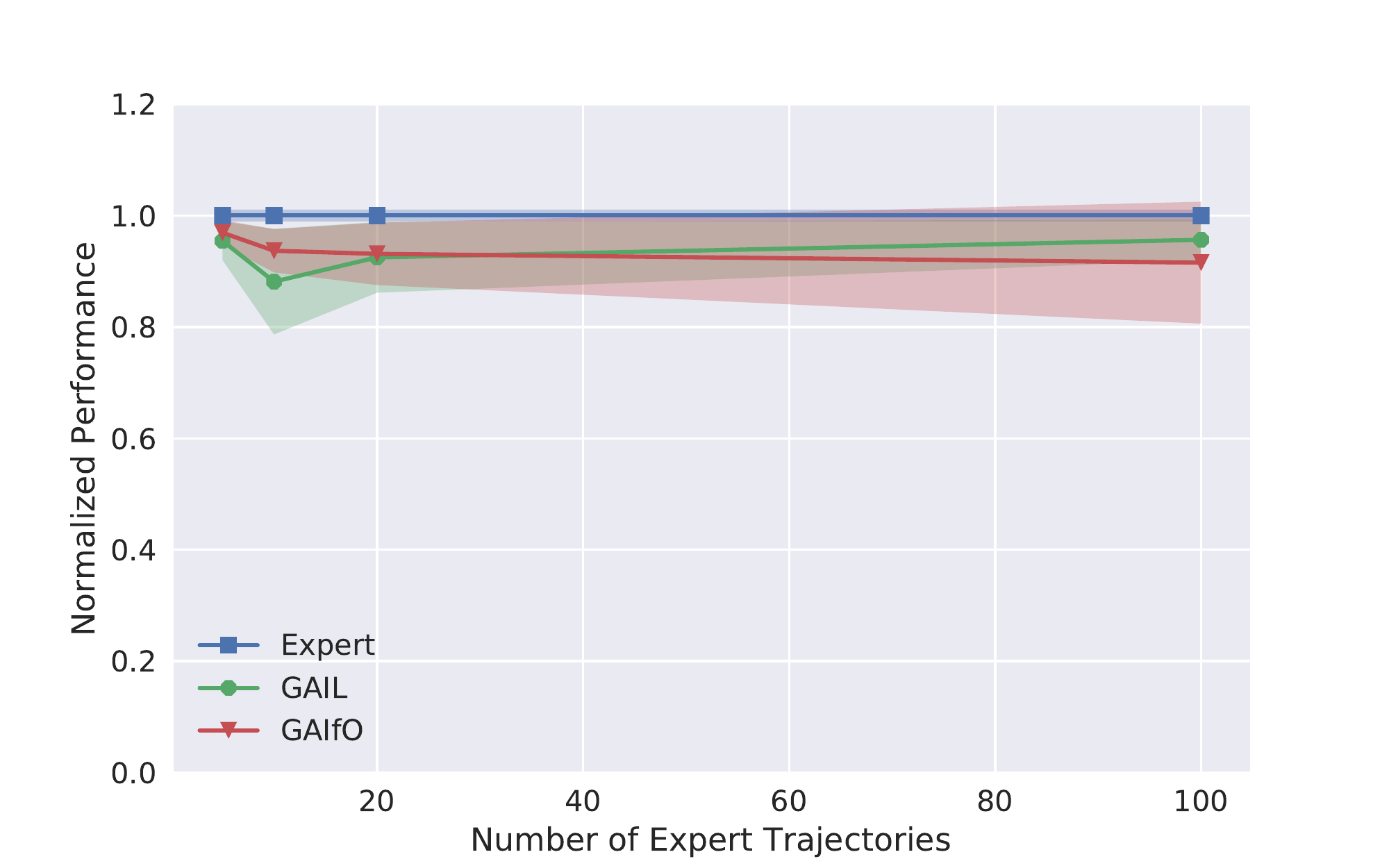}
    }
    \subfigure[Walker2d-v2]{
    \includegraphics[width=0.3\textwidth]{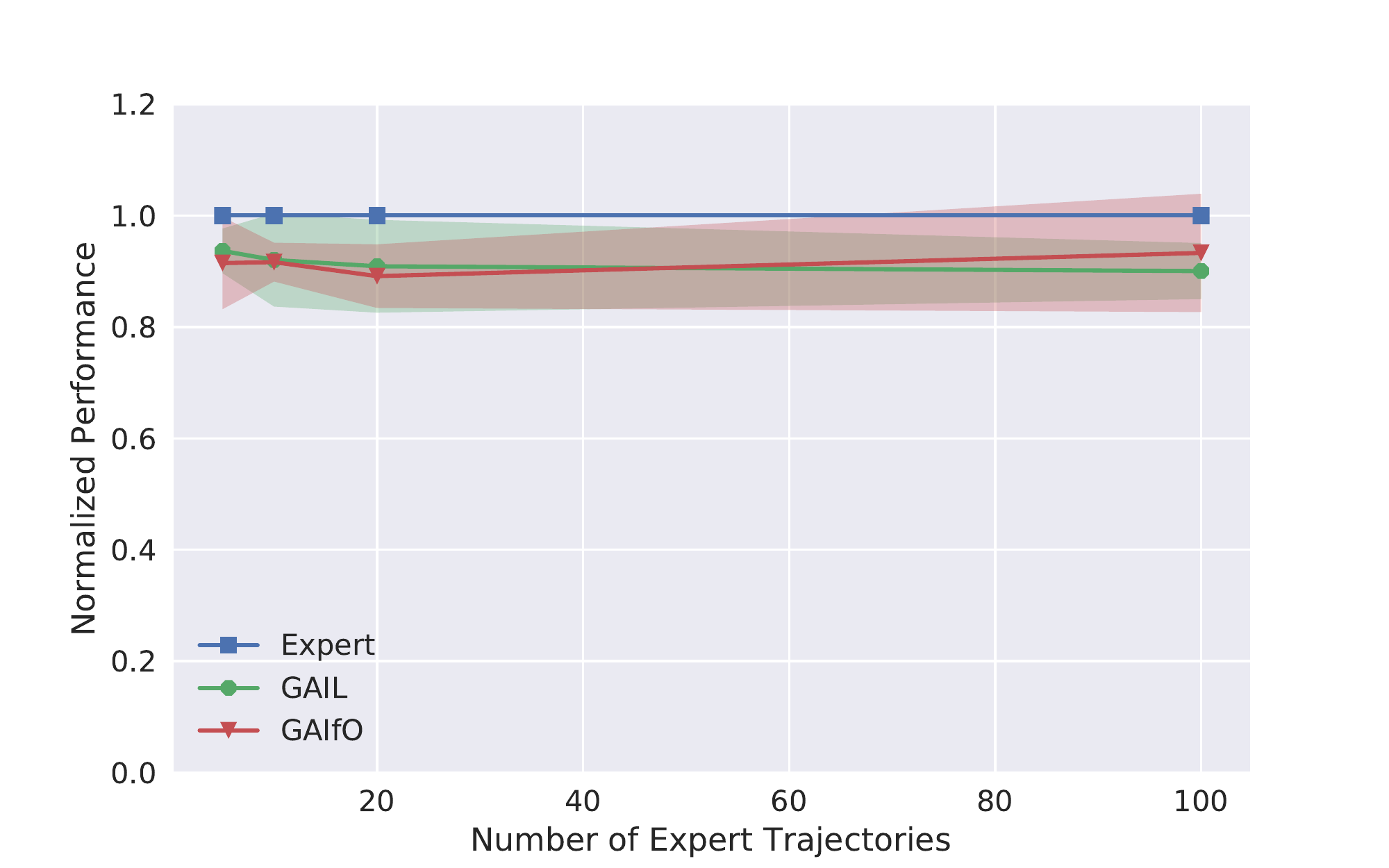}
    }
    \subfigure[HalfCheetah-v2]{
    \includegraphics[width=0.3\textwidth]{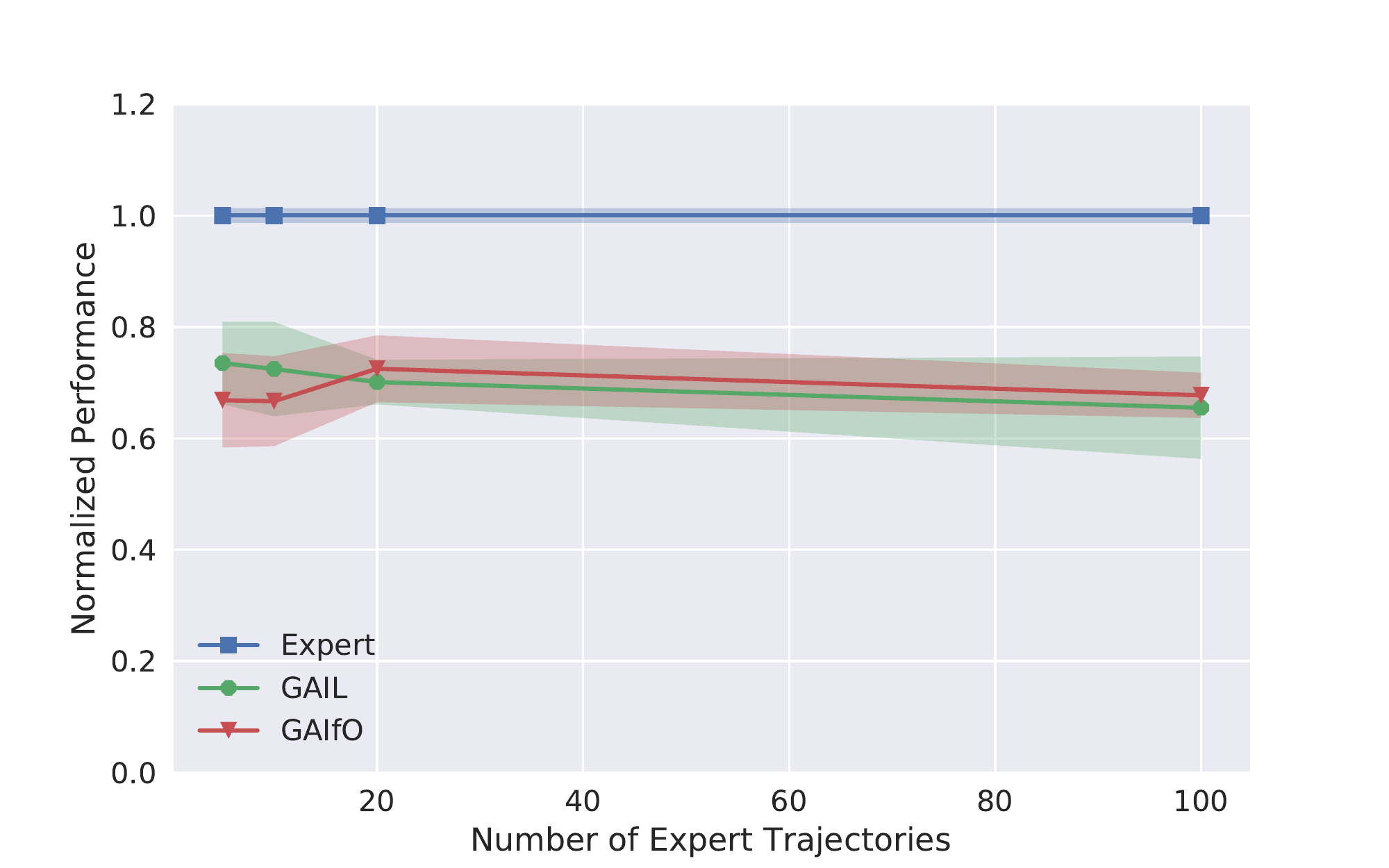}
    }
    \subfigure[Humanoid-v2]{
    \includegraphics[width=0.3\textwidth]{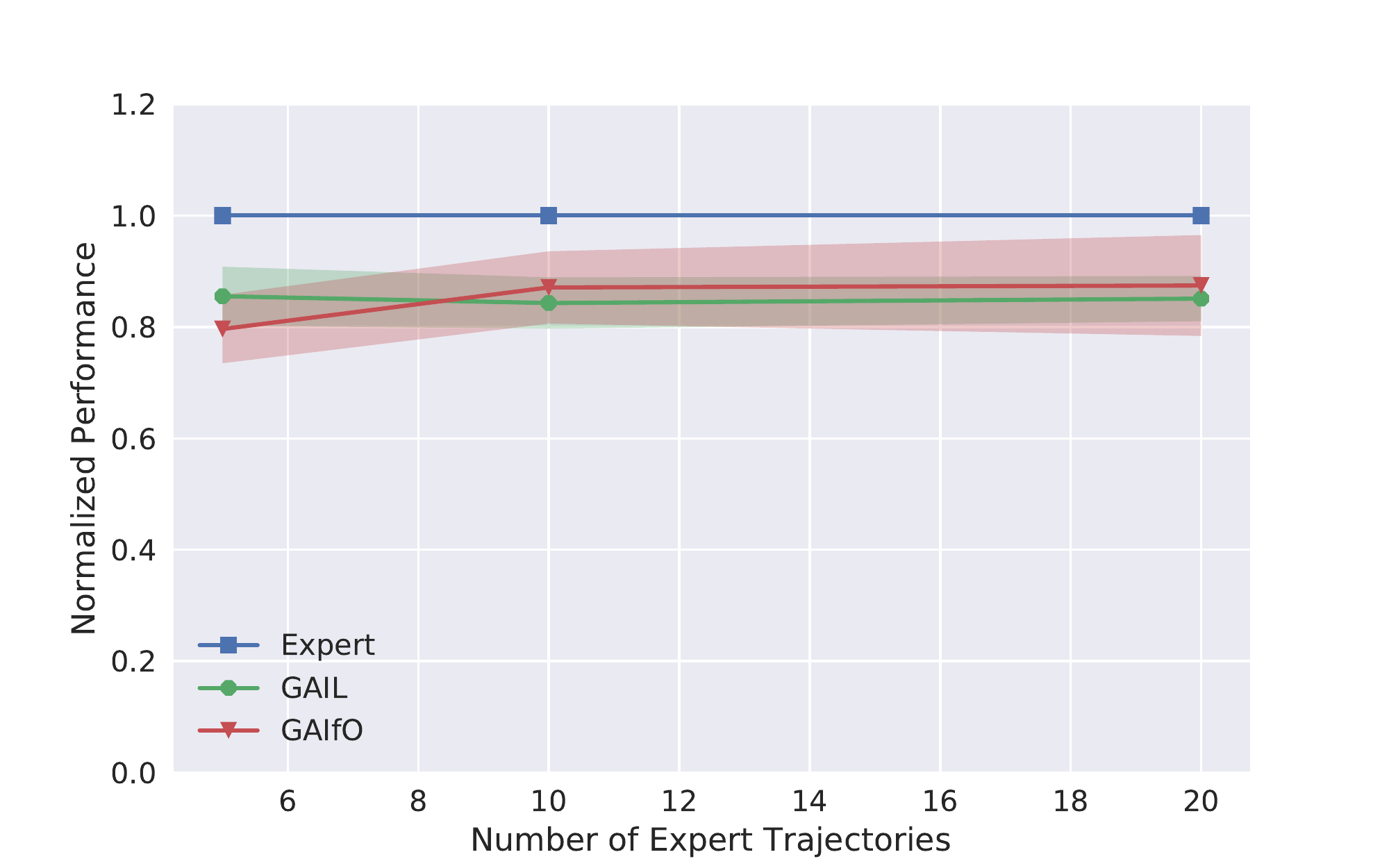}
    }
\caption{Impact of the number of expert trajectories on the performance of GAIL and GAIfO. The performance has been normalized by dividing by the expert performance.}
\label{impactoftraj}
\end{figure*}

\subsection*{D.3 Impact of Input Normalization and Spectral Normalization}
In the paper, we demonstrate that both input normalization and spectral normalization help to improve the performance of GAIL and GAIfO with only one specific environment. In this subsection, we compare the performance of GAIL and GAIfO with/without input normalization or spectral normalization in several additional environments. The results are shown in Fig. \ref{inputnorm}-\ref{spectralnorm}. It can be seen that without input normalization, the performance of GAIL and GAIfO degenerates dramatically. Specifically, in Walker2d-v2, a vast performance gap between GAILw/oIN and GAIfOw/oIN is observed. In terms of spectral normalization, both GAIfO-sn and GAIL-sn outperform the un-spectral normalized versions, \emph{i.e.}, GAIfO and GAIL. In other words, spectral normalization helps to take full advantage of these generative adversarial imitation learning algorithms. 

\begin{figure}[h]
\centering
    \subfigure[Walker2d-v2]{
    \includegraphics[width=0.35\textwidth]{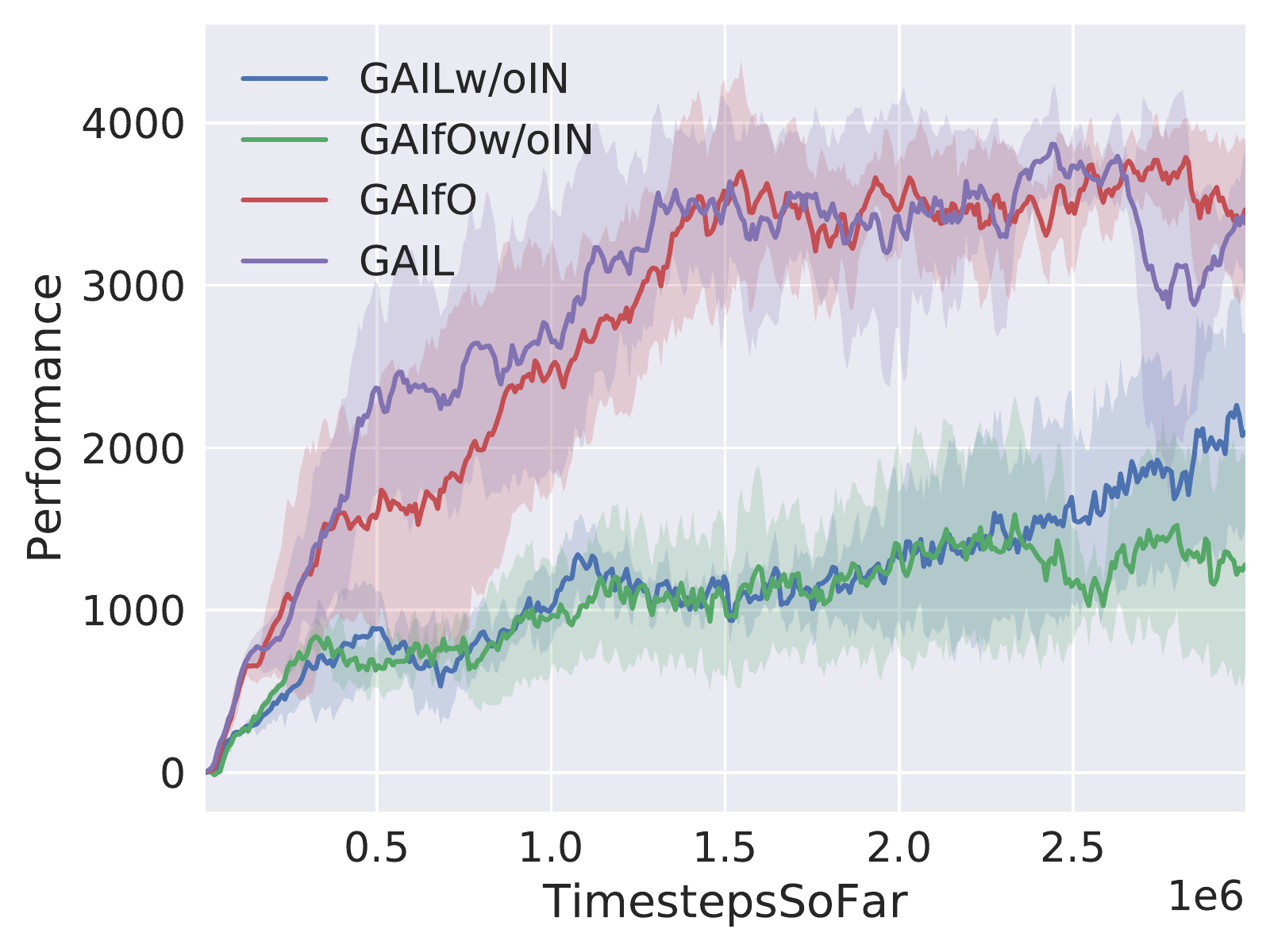}
    }
    \subfigure[HalfCheetah-v2]{
    \includegraphics[width=0.35\textwidth]{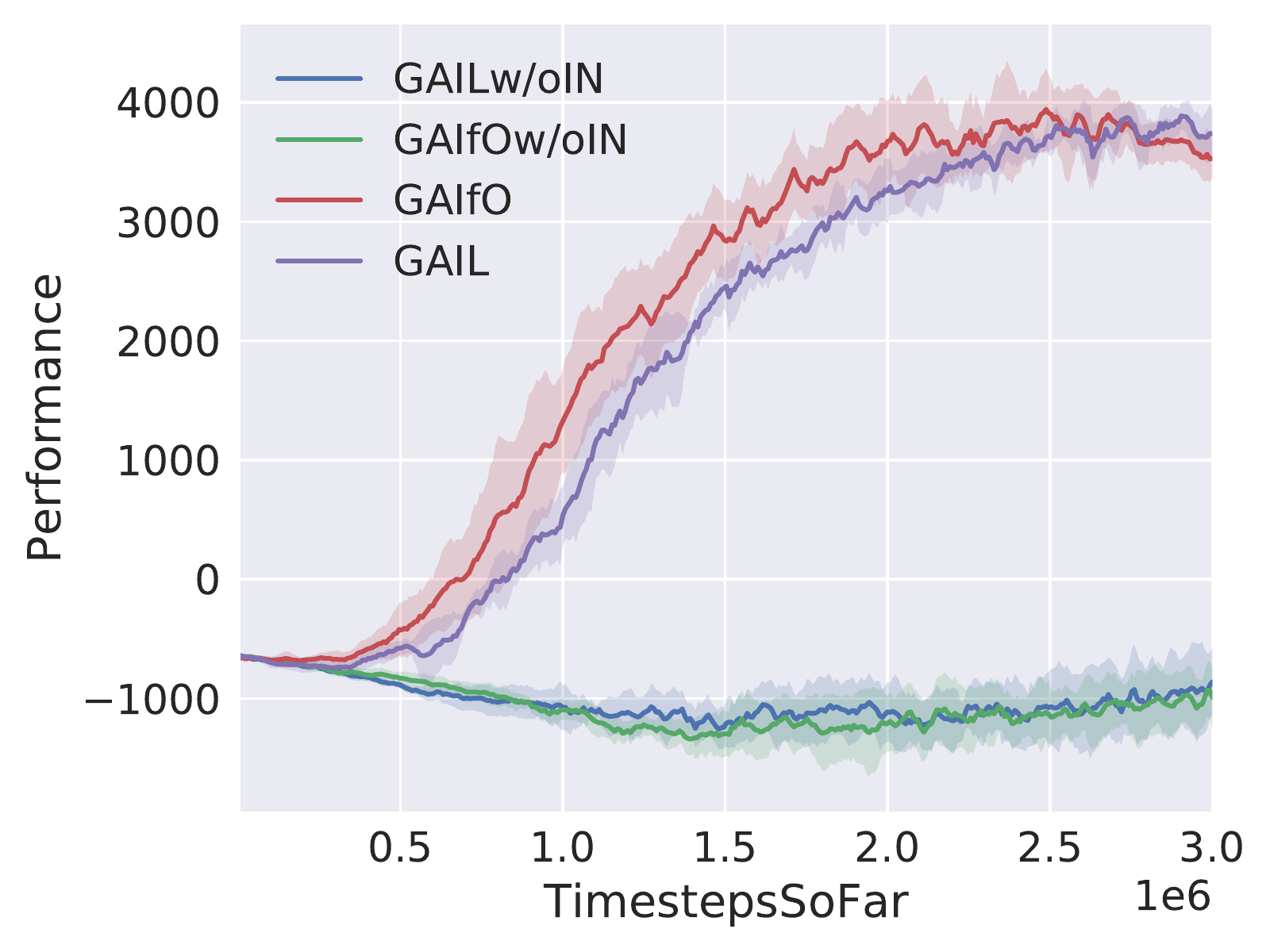}
    }
\caption{Impact of input normalization.}
\label{inputnorm}
\end{figure}

\begin{figure}[h]
\centering
    \subfigure[Walker2d-v2]{
    \includegraphics[width=0.35\textwidth]{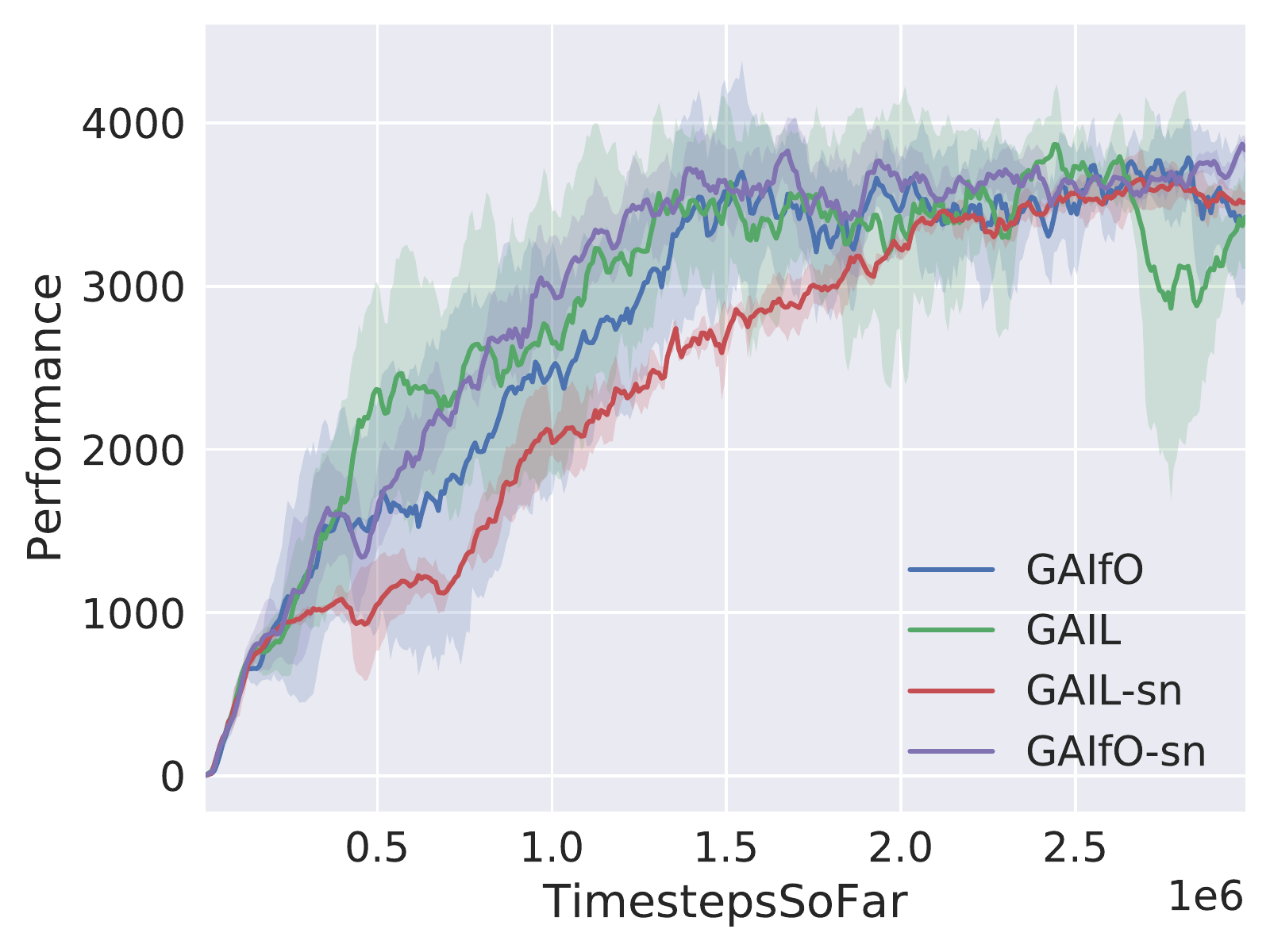}
    }
    \subfigure[Humanoid-v2]{
    \includegraphics[width=0.35\textwidth]{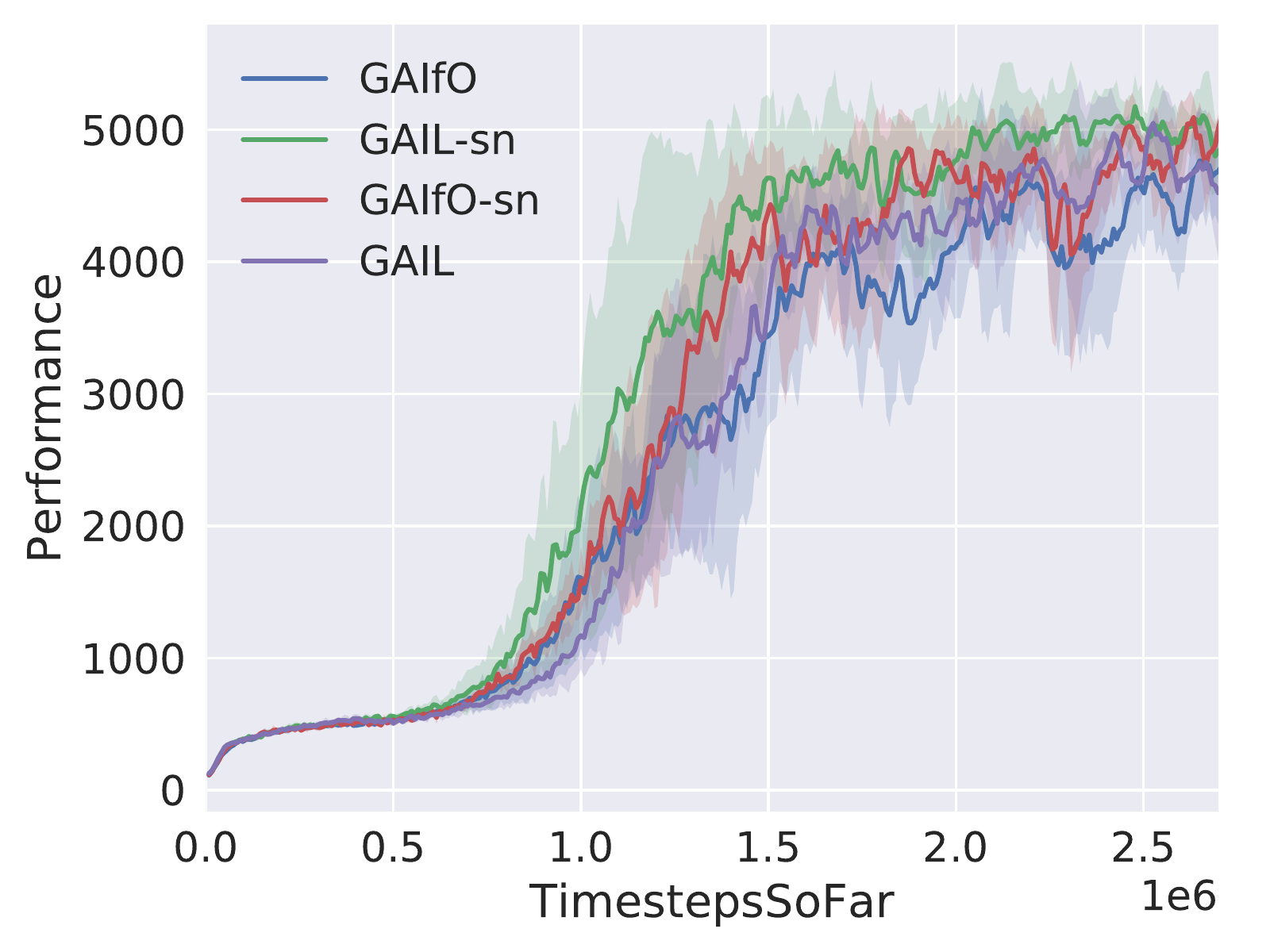}
    }
\caption{Impact of spectral normalization.}
\label{spectralnorm}
\end{figure}


\subsection*{D.4 Third Party Verification}
Our experiment results illustrate that GAIfO is able to achieve comparable performance to GAIL in deterministic robot environments. In addition, we adopt a third party implementation of GAIL and GAIfO \cite{ota2020tf2rl} that is available online to make further efforts to demonstrate the almost equivalence between GAIL and GAIfO. This third party implementation is called TF2RL and the authors provide implementations of deep reinforcement learning algorithms using TensorFlow 2.x.. Specifically, GAIL and GAIfO, including corresponding spectral normalized versions, GAIL-sn and GAIfO-sn, are included in this repository. To the best of our knowledge, this repository is the only one, which provides the open-sourced implementation of GAIfO. \par 
There are three major differences lying between our implementation and the third party implementation. First, we use RL algorithm TRPO \cite{schulman2015trust} to serve as the generator while they employ DDPG \cite{lillicrap2015continuous} to train the policy. Second, their repository does not support parallel running for imitation learning algorithms. This means that the time to get the result of an experiment is relatively long. In contrast, our code uses Message Passing Interface (MPI) to speed up the training process. Third, spectral normalization plays a comparatively important role in their GAIL and GAIfO. We find that their GAIL and GAIfO are likely to fail to imitate the expert without spectral normalization and generate NaN (Not a Number) during training. On the contrary, our implementation is able to train the policy stably and imitate the expert well even without spectral normalization.\par

In the experiments, we use default parameters in their implementation, which are listed in Table. \ref{hyperparam2}. As discussed above, their GAIL and GAIfO sometimes may fail to imitate the expert. So, we just compare the performance of their spectral normalized algorithms. As noted in the paper, the introduction of spectral normalization does not affect the core of LfO and LfD. It helps to take full advantage of LfO and LfD and does not affect the optimizing targets of LfO and LfD. The HalfCheetah-v2 environment is chosen to test the performance of GAIL and GAIfO in TF2RL. And we run the same experiment for 5 times to fairly evaluate  GAIL and GAIfO since they do not provide the interface to control the seed. The task uses 20 expert trajectories to conduct imitation learning. The learning curves using their code are shown in Fig. \ref{fig:learningcurvetf2rl}. From Fig. \ref{fig:learningcurvetf2rl}, we can conclude that both GAIL and GAIfO can achieve the expert-level performance, and the performance of GAIfO is comparable to that of GAIL. 

\begin{table}
\caption{Hyper-parameters in TF2RL.}
\centering
\begin{tabular}{l|l}
 \hline
 Hyper-parameters& Value\\
 \hline
 Common parameters & \quad\\
 \qquad Network size (for all networks)   & (400,300)  \\
 \qquad Activation & $ReLU$\\
 \qquad Optimizer  & Adam \\
 \qquad Generator net update times & 1 \\
 \qquad Discriminator net update times & 1 \\
 \hline
 DDPG parameters & \quad\\
 \qquad Discount & 0.99\\
 \qquad Replay buffer size  &  $1 \times 10^{6}$\\
 \qquad Warmup steps  &  $1 \times 10^{4}$\\
 \qquad Learning rate (Value net and policy net)  & $1 \times 10^{-3}$ \\
 \qquad Batch size  & 100 \\
 \hline
 Discriminator parameters & \quad\\
 \qquad Learning rate (Discriminator net) & $1 \times 10^{-3}$ \\
  \qquad Batch size  & 32 \\
 \hline
\end{tabular}
\label{hyperparam2}
\end{table}

\begin{figure}
\centering
\includegraphics[width=0.4\textwidth]{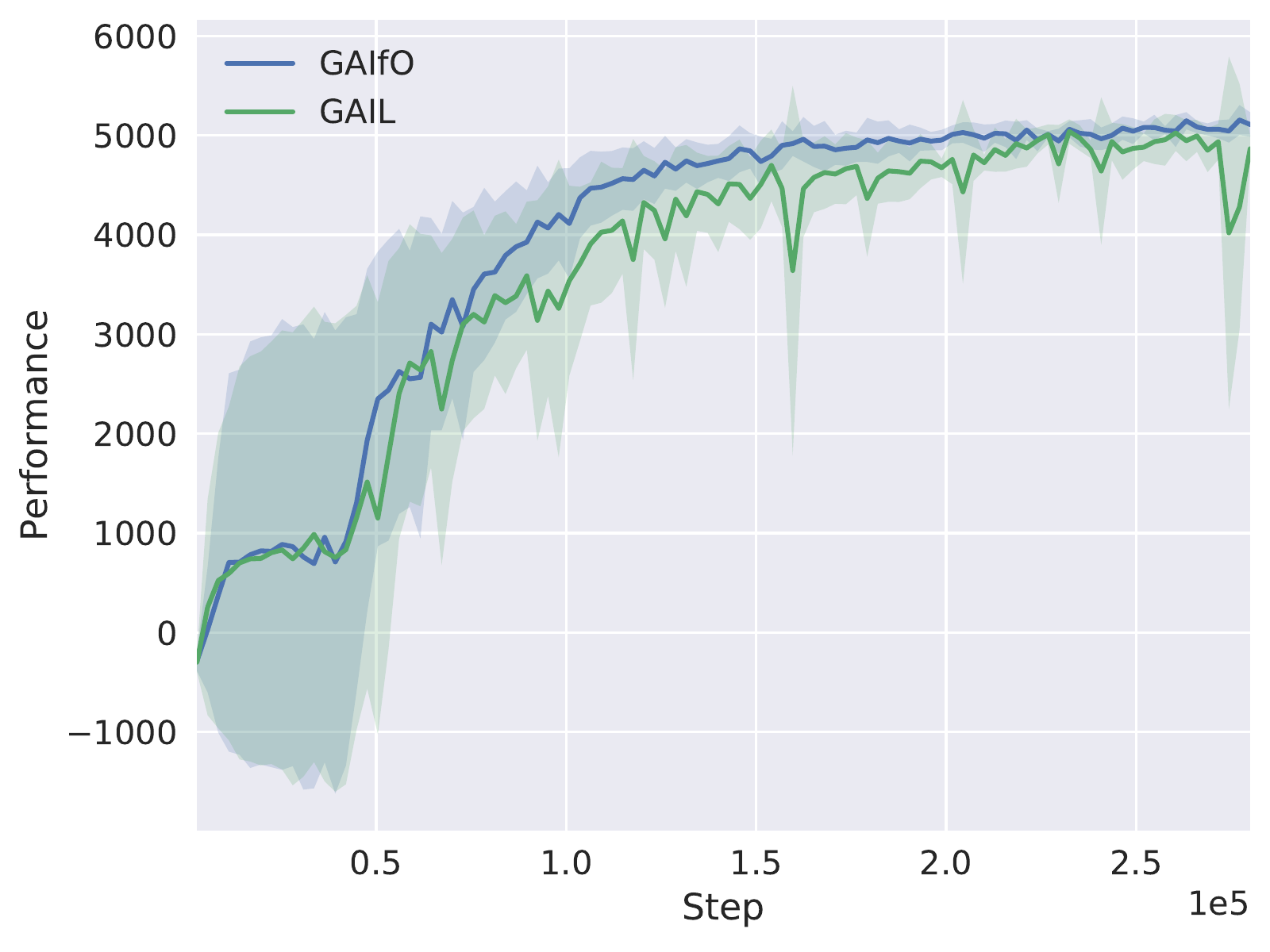}
\caption{Learning curves of GAIL and GAIfO using TF2RL in HalfCheetah-v2.}
\label{fig:learningcurvetf2rl}
\end{figure}

\end{document}